\documentclass{article}

\usepackage{arxiv}

\usepackage{kpfonts}
\usepackage[utf8]{inputenc} 
\usepackage[T1]{fontenc}    
\usepackage{hyperref}       
\usepackage{url}            
\usepackage{booktabs}       
\usepackage{amsfonts}       
\usepackage{nicefrac}       
\usepackage{microtype}      
\usepackage{lipsum}		
\usepackage{graphicx}
\usepackage{natbib}
\usepackage{doi}

\usepackage{algorithm}
\usepackage{algorithmic}

\usepackage{amsfonts} 
\usepackage{amsmath} 
\usepackage{amsthm} 
\usepackage{bm} 
\usepackage{dsfont} 
\usepackage[inline]{enumitem} 
\usepackage{subcaption} 
\usepackage[table,xcdraw]{xcolor} 
\usepackage{cleveref}
\usepackage{booktabs}
\usepackage{multirow} 

\usepackage{tikz} 
\usetikzlibrary{bayesnet}
\usetikzlibrary{patterns} 

\newcommand{\indep}{\perp \!\!\! \perp} 
\newcommand{\dep}{\not\!\perp\!\!\!\perp} 

\newtheorem{theorem}{Theorem}[section] 
\newtheorem{corollary}[theorem]{Corollary}
\newtheorem{lemma}[theorem]{Lemma}
\newtheorem{proposition}{Proposition}

\newtheorem{definition}{Definition}

\newtheorem{example}{Example}

\newcommand{\figTradeOff}{Figure 4}
\newcommand{\apxBestResponse}{Appendix A.1}
\newcommand{\apxInvalidIV}{Appendix B}
\newcommand{\apxMaximin}{Appendix C}
\newcommand{\apxBoundedReductionMulti}{Appendix D}
\newcommand{\apxExperiments}{Appendix F}
\newcommand{\apxExpSetup}{Appendix F.1}
\newcommand{\apxAnalysesOfAssumptions}{Appendix F.2}
\newcommand{\apxFurtherExp}{Appendix F.3}
\newcommand{\apxBoundingTheta}{Appendix F.4}

\title{Causal Strategic Learning with Competitive Selection}

\author{ 
Kiet Q. H. Vo\textsuperscript{\rm 1,\rm 2}, Muneeb Aadil\textsuperscript{\rm 1,\rm 2}, Siu Lun Chau\textsuperscript{\rm 1}, Krikamol Muandet\textsuperscript{\rm 1}
\\ \\
\textsuperscript{\rm 1}CISPA Helmholtz Center for Information Security, Saarbr\"ucken, Germany \\
\textsuperscript{\rm 2}Saarland University, Saarbr\"ucken, Germany\\
}

\renewcommand{\shorttitle}{Causal Strategic Learning with Competitive Selection}

\def\thetaao{\bm{\theta}^{\text{AO}}}
\def\thetaaohat{\hat{\bm{\theta}}^{\text{AO}}}

\def\thetastar{\bm{\theta}^*}
\def\thetastarhat{\hat{\bm{\theta}}^*}

\def\thetaolshat{\hat{\bm{\theta}}^{\text{OLS}}}

\def\X{X} 
\def\B{B} 
\def\homomatrix{\bm{\mathcal{E}}}
\def\vecalpha{\bm{\alpha}} 

\begin{document}
\maketitle

\begin{abstract}
We study the problem of agent selection in causal strategic learning under multiple decision makers and address two key challenges that come with it. 
Firstly, while much of prior work focuses on studying a fixed pool of agents that remains static regardless of their evaluations, we consider the impact of selection procedure by which agents are not only evaluated, but also selected.
When each decision maker unilaterally selects agents by maximising their own utility, we show that the optimal selection rule is a trade-off between selecting the best agents and providing incentives to maximise the agents' improvement. 
Furthermore, this optimal selection rule relies on incorrect predictions of agents' outcomes. 
Hence, we study the conditions under which a decision maker's optimal selection rule will not lead to deterioration of agents' outcome nor cause unjust reduction in agents' selection chance. 
To that end, we provide an analytical form of the optimal selection rule and a mechanism to retrieve the causal parameters from observational data, under certain assumptions on agents' behaviour. 
Secondly, when there are multiple decision makers, the interference between selection rules introduces another source of biases in estimating the underlying causal parameters. 
To address this problem, we provide a cooperative protocol which all decision makers must collectively adopt to recover the true causal parameters. 
Lastly, we complement our theoretical results with simulation studies.
Our results highlight not only the importance of causal modeling as a strategy to mitigate the effect of gaming, as suggested by previous work, but also the need of a benevolent regulator to enable it.
\end{abstract}

\keywords{Causal Inference \and Strategic Learning}

\section{Introduction}
Machine Learning (ML) has gained significant popularity in facilitating personalised decision making across diverse domains such as healthcare \citep{Wiens19:Roadmap,chau2021bayesimp,Ghassemi22:ML-Health}, criminal justice \citep{kleinberg2018human}, college admissions \citep{harris2022strategic},  hiring \citep{Deshpande20:Ai-Resume}, and credit scoring \citep{bjorkegren2020behavior}. In these critical domains, mutual trust between decision makers and agents who are affected by the decisions is of utmost importance. 
As a result, the decision makers might need to render algorithmic rules transparent to all stakeholders. However, this transparency can incentivise agents to strategically adjust their variables to receive more favorable decisions, resulting in either genuine improvements or gaming~\citep{pmlr-v130-bechavod21a}. Although in both scenarios agents receive better decision outcomes, gaming is undesirable for the decision makers as it negatively impacts their utility. Learning under strategic behavior is well-studied in both economics and machine learning \citep{hardt2016strategic,perdomo2020performative,dranove2003more,dee2019causes,munro2022learning}. Our work aligns with research efforts to identify causal features that reduce gaming effects and to promote genuine agent improvements \citep{miller2020strategic}, an approach often referred to as causal strategic learning~(CSL).

Let us consider a college admission example from \citet{harris2022strategic}. 
The college, acting as the \textit{decision maker~(DM)}, aims to evaluate applicants~(\textit{agents}) by predicting their prospective college GPAs based on their submitted high school GPAs and SAT scores. For transparency, the college makes this evaluation rule public. In response, applicants can strategically direct their efforts on certain exams (high school or SAT) to optimise their evaluations. Recognising this strategic approach, the college's objective is to formulate and publicise an evaluation rule that maximises the expected college GPA (or \textit{agents' outcome}) for all applicants. Envision a scenario where a student's college GPA is causally determined by their high school GPA only, yet the deployed rule considers both exam results. There is potential for gaming behavior under this rule, if an applicant emphasises their SAT preparation over their high school GPA,
since this might boost their evaluation without necessarily improving the actual college academic performance. 

\begin{figure}[t!]
\centering
\includegraphics[width=0.85\columnwidth]{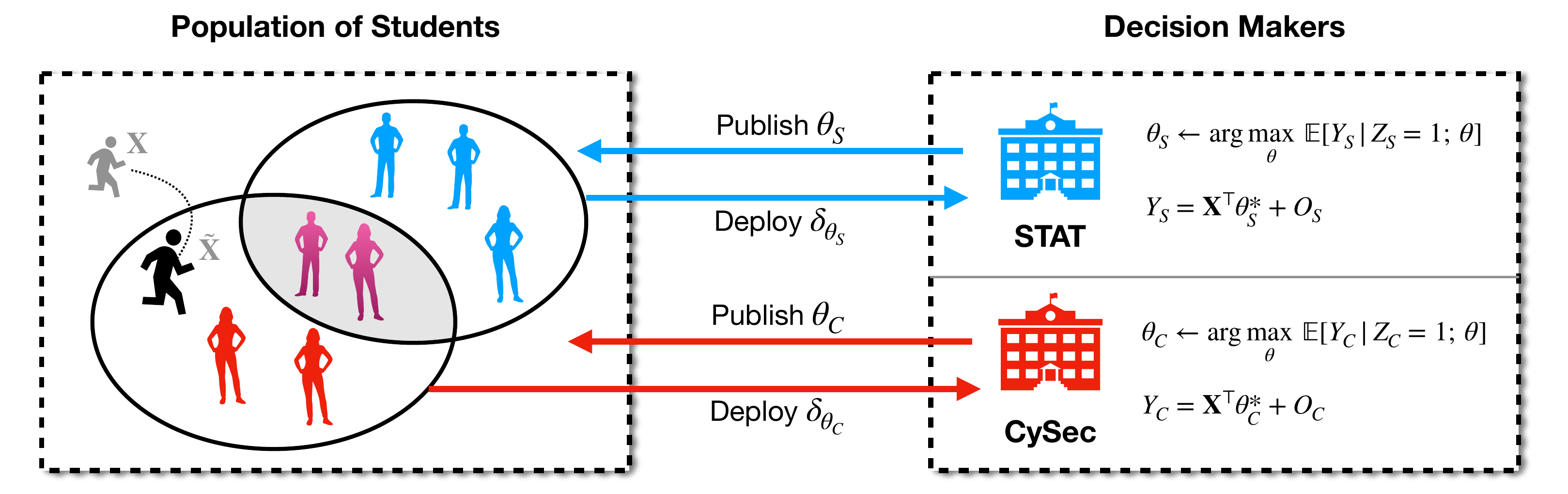}
\caption{Causal strategic learning with two decision makers. The statistics (STAT) and cyber-security (CySec) departments, each applies selection rules $\bm{\delta}_{\theta_S}$ and $\bm{\delta}_{\theta_C}$ to maximise their enrolled students' expected outcomes (future GPA). The parameters $\theta_S$ and $\theta_C$ are made public, prompting future students to strategically alter their attributes ($\mathbf{X}$ to $\tilde{\mathbf{X}}$) to boost admission chances.}
\label{fig:selection-diagram}
\end{figure}

The above example underscores the necessity of incorporating causal knowledge into decision making to incentivise agents towards genuine improvement, aligning with what \citet{miller2020strategic} have proven. CSL presents numerous challenges. 
For example, \citet{alon2020multiagent} explore mechanism designs that incentivise agents to respond with the intended outcomes of the DM, assuming knowledge of the true underlying causal structure. Similarly, \citet{munro2022learning} also assumes knowledge of casual information and incorporates stochasticity into their released decision rule to discourage gaming. However, without prior causal knowledge, learning the true causal mechanism in practice is challenging due to confounding bias in observational data. To address this, \citet{shavit2020causal} show that a DM can publicise a sequence of evaluation rules specifically to eliminate confounding bias and achieving causal identifiability. 
In contrast, \citet{harris2022strategic} consider scenarios in which the DM can utilise the evaluation rule itself as instrumental variable, and identify the true causal mechanism via instrumental variable regression~\citep{Angrist96:IV,Newey03:NIV,Hartford17:DIV,singh2019kernel,muandet2020dual}. 
While much of previous CSL research focuses on evaluating (and motivating) agents in light of strategic feedback from a single DM's perspective, our research extends further, considering not just evaluating, but also selecting agents based on their evaluations. This brings in additional challenges, notably the introduction of selection bias, which undermines previous causal identifiability results. Additionally, we venture into situations with multiple DMs competing to select agents. We believe this work is well-motivated for real-world strategic learning scenarios that involve competitive selection, such as in hiring and loan application.

Continuing from our motivating example, consider that we now have multiple college departments (as DMs), e.g. statistics and cyber-security, competing not only to evaluate applicants but also to select them based on their evaluations (see \Cref{fig:selection-diagram}). 
Unlike previous methods, each department (DM) aims to optimise the expected GPA of their enrolled students, rather than focusing on all applicants.
This natural objective nonetheless leads to a dilemma between selecting the top-performing candidates and motivating general candidates to improve. Furthermore, a selection rule focusing solely on top candidates can disincentivise self-improvement, potentially lowering future college GPAs~(see \Cref{col:max-improvement-s}). Additionally, as the optimal selection rule has to rely on incorrect (non-causal) predictions of agents' outcomes, their chances of being selected can be diminished compared to if evaluations were based on accurate (causal) predictions~(see \Cref{col:bounded-reduction-s}). We refer to an agent's prospective outcome and selection chances collectively as \textit{agent welfare}. To safeguard such welfare, we adopt a regulator's viewpoint, proposing regulations for the DM to follow, such that their resulting optimal decision rule will lead to neither deterioration of agents' outcomes nor excessive reduction in agents' selection chance. As such regulation requires DM to have access to causal parameters, we provide conditions for a single DM to achieve causal identifiability under selection bias. With multiple DMs, the selection bias is now harder to correct for due to the interference between decision rules. In particular, it is difficult for any individual DM to predict an agent's strategic response when that agent is incentivised by all DMs. Additionally, anticipating their compliance behavior is challenging since this agent can adhere to at most one DM's positive decision. Consequently, we propose a cooperative protocol for the DMs to follow so that their causal parameters can be identified, to subsequently safeguard the welfare of agents.

The rest of the paper is outlined as follows. \Cref{sec:prob-formulation} introduces the CSL formulation with selection procedure under multiple DMs. 
\Cref{sec:main-results} then discusses the impact of selection in the context of CSL alongside our main results and extensions to the setting of competitive selection. We validate our approach through various simulation studies in Section \ref{sec:experiments}. Finally, we conclude in Section \ref{sec:conclusion}. All proofs are provided in the appendices.

\section{Causal Strategic Learning with Selection}
\label{sec:prob-formulation}

\textbf{Notations. } We denote random variables and random vectors with upper case letters, and their realisations with lower case and bold lower case letters, respectively. Random matrices are also denoted with upper case letters, and their realisations with bold upper case. We write $\{1,\ldots, n\}$ as $[n]$.

Following prior work \citep{shavit2020causal,harris2022strategic,bechavod2022information}, we build our setting on the sequential decision making context, following the framework of Stackelberg game. We assume throughout that there exist $n$ decision makers~(DMs), with $n\geq 1$, who take turn with agents playing their strategies over $T$ rounds indexed by $t \in [T]$. Let $W_{\mathit{it}}$ be a binary variable representing the decision from DM $i$ for the sole agent who arrives at round $t$, e.g., whether or not the college $i$ admits this student. 
At the beginning of each round, each DM publicises their decision rule $\bm{\delta}_{\bm{\theta}_{\mathit{it}}}$ parameterised by the parameter vector $\bm{\theta}_{\mathit{it}}\in\mathbb{R}^m$, i.e.,
\begin{equation*}
\bm{\delta}_{\bm{\theta}_{\mathit{it}}}:\mathbf{x}\mapsto p\left(W_{\mathit{it}}=1\mid\X_t=\mathbf{x} \,;\, \bm{\theta}_{\mathit{it}}\right), \quad i\in[n]
\end{equation*}
where $\X_t\in\mathbb{R}^m$ denotes the random vector containing the covariates of the agent in round $t$ and $p(W_{it}=1\,|\,\X_t=\mathbf{x} \,;\, \bm{\theta}_{it})$ is a probability that this agent will later receive a positive decision, i.e., being admitted into the college, if they report attributes $\X_t=\mathbf{x}$. 
We assume that $W_{\mathit{it}}\sim\text{Bernoulli}(\bm{\delta}_{\bm{\theta}_{\mathit{it}}}(\X_t))$. 
After knowing about $\{\bm{\delta}_{\bm{\theta}_{it}}\}_{i=1}^n$, this agent modifies their attributes and then reports the final values $\mathbf{x}$, e.g., SAT score and high school GPA, to all DMs, so as to maximise the chance of receiving favorable decisions.
Next, all DMs evaluate this agent using their decision rules and return the selection statuses $\{w_{it}\}_{i=1}^n$. Finally, the agent's compliance to the decisions can be modeled as a random variable $Z_t\sim\text{Categorical}(\{0\}\cup[n])$, whose value dictates which positive decision the agent will comply with\footnote{When $Z_t=0$, the agent either does not comply with any of the positive decisions or does not receive any positive decision.}.
Throughout this work, we focus on the \emph{perfect information} setting where both DMs and agents know all information about the decision rules including their parameter vectors \citep{shavit2020causal,harris2022strategic}. Specifically, for round $t_s$, the agent knows about $\{\bm{\delta}_{\bm{\theta}_{it_s}},\bm{\theta}_{it_s}\}_{i=1}^n$ and all DMs know about $\{\{\bm{\delta}_{\bm{\theta}_{it}},\bm{\theta}_{it}\}_{i=1}^n\}_{t=1}^{t_s-1}$.

Following \citet{harris2022strategic}, we assume that the potential outcome of an agent, $Y_{\mathit{it}}\in\mathbb{R}$, e.g., their future GPA, in any environment $i$ is a linear function of their covariates: $Y_{\mathit{it}}:=\X^{\top}_t\bm{\theta}_i^*+O_{\mathit{it}}$ where $\bm{\theta}_i^*\in\mathbb{R}^m$ is the true causal parameter vector that maps the covariates $\X_t$ to the outcome $Y_{\mathit{it}}\in\mathbb{R}$ and $O_{\mathit{it}}$ is the unobserved noise.
In practice, the DMs lack access to the true $\bm{\theta}^*_i$, so each of them bases their decision on the predicted outcome $\hat{y}_{it}=\mathbf{x}^\top\bm{\theta}_{it}$ using the agent's covariates $\X_t=\mathbf{x}$ where $\bm{\theta}_{it}$ is a parameter estimate. 
Finally, we assume that the covariates $\X_t$ is a linear function of an agent's baseline and their strategic improvement, namely $\X_t:=\B_t+\mathcal{E}_t\mathbf{a}_t$ where the conversion matrix $\mathcal{E}_t\in\mathbb{R}^{m\times d}$ translates their strategic action $\mathbf{a}_t\in\mathbb{R}^d$ into the improvement upon the baseline $\B_t\in\mathbb{R}^m$.
The unobserved noise $O_{it}$ is correlated with the agent's baseline $\B_t$ and is specific to the environment $i$, which can be due to the private type of each agent, e.g., a student's socioeconomic background, that can further influence their academic baseline $\B_t$ and their cultural fit $O_{it}$ in this environment.

\textbf{Agents' utilities. } Since each agent has access to multiple predicted outcomes $\hat{y}_{it}$ (where $i\in[n]$) alongside their preferred environments, we assume that the agent $t$ aims at maximising the following utility function 
\begin{equation}\label{eq:agent-utility} 
    u(\mathbf{a}_t) := \sum^n_{i=1}\gamma_{it}\hat{y}_{it}({\bf a}_t)
    -\frac{1}{2}\|\mathbf{a}_t\|^2_2  \quad \text{with } \gamma_{it} \geq 0,\; \forall i,t
\end{equation}
in each round $t$ after being informed of the parameter vectors, where $\{\gamma_{it}\}_{i=1}^n$ represents the preference of this agent.
Unlike previous work \citep{shavit2020causal,harris2022strategic,bechavod2022information}, the utility function \eqref{eq:agent-utility} also involves the agent's preference over multiple DMs.
For any list of parameter vectors $\{\bm{\theta}_{1t},\bm{\theta}_{2t},\ldots,\bm{\theta}_{nt}\}$, it is not difficult to see that the maximiser of \eqref{eq:agent-utility} is $\mathbf{a}_t=\mathcal{E}^\top_t\left(\sum_{i=1}^n\gamma_{it}\bm{\theta}_{it}\right)$; see \apxBestResponse\ for the full proof.

\textbf{Decision makers' objectives. } We assume that the DMs are utility maximisers each of whom aims to maximise the expected future outcome of the agents that comply with their decisions. Without loss of generality, we specify the objective function for an arbitrary DM $i$:
\begin{align}\label{eq:dm-objectives}
\max_{\bm{\theta}_{\mathit{it}}}\; &\mathbb{E}\left[Y_{\mathit{it}}(\{\bm{\theta}_t^{-i},\bm{\theta}_{\mathit{it}}\})\mid Z_t=i\ ;\ \bm{\theta}_{\mathit{it}}\right]
\end{align}
where we use $\{\bm{\theta}_t^{-i},\bm{\theta}_{it}\}$, $\bm{\theta}_t^\text{all}$, or $\{\bm{\theta}_{it}\}_{i=1}^n$ to denote a collection of parameters associated with the deployed selection rules. 
We use the notation $Y_{it}(\{\bm{\theta}_t^{-i},\bm{\theta}_{it}\})$ to highlight that the outcome variable is a function of all parameters $\bm{\theta}_t^\text{all}$ due to agents' strategic behaviour. Furthermore, notice that the expectation also depends on the conditional distribution of the rival DMs' parameters, $p(\bm{\theta}_t^{-i}|Z_t=i,\bm{\theta}_{it})$. More detailed discussion will follow in subsequent sections.

In summary, our approach distinguishes itself from previous work in causal strategic learning mainly by its integration of the selection variable $W_{\mathit{it}}$ within a competitive context involving multiple DMs. \Cref{fig:causal-graphs} illustrates the causal graphs associated with our novel setting.

\begin{figure}
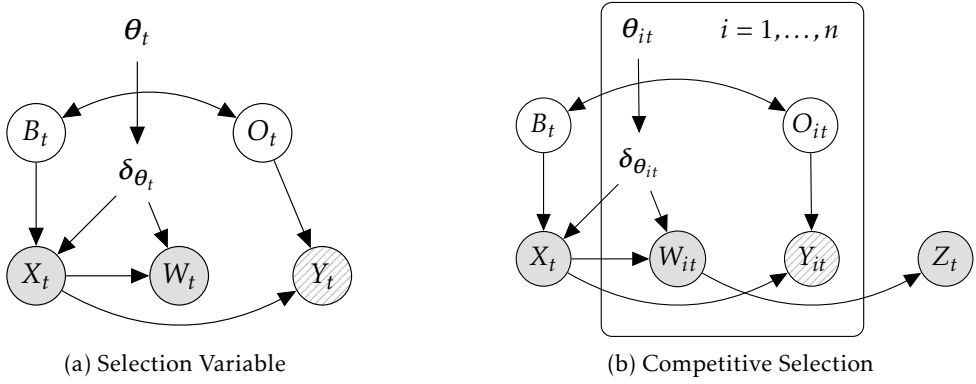

  \begin{subfigure}{0.45\textwidth}  
  \centering
  \resizebox{!}{0.6\textwidth}{
  \tikz{
     \node[obs] (vx) {$\X_t$};%
     \node[latent,above=of vx] (vb) {$\B_t$}; 
     \node[latent,draw opacity=0, above right = of vb] (vtheta) {$\bm{\theta}_t$};
     \node[obs, right=of vx] (vw) {$W_t$};     
     \node[latent,draw opacity=0, above right = of vx] (vdelta) {$\bm{\delta}_{\bm{\theta}_t}$};
     \node[latent, right =of vb, xshift=1cm] (vO) {$O_t$};
     \node[latent, below=of vO, right = of vw, pattern=north east lines, pattern color=gray!50] (vy) {$Y_t$};
     
     \edge {vb, vdelta} {vx} ;
     \edge {vx} {vw};
     \path (vx) edge[bend right, ->] (vy);
     \edge {vtheta} {vdelta};
     \edge {vdelta} {vw};
     \path (vb) edge[bend right, in=150, out=30, <->] (vO); 
     \edge {vO} {vy};
     }}
     \caption{Selection Variable}
     \label{fig:causal-single-env}
  \end{subfigure}
  \begin{subfigure}{0.45\textwidth}  
    \centering
    \resizebox{!}{0.6\textwidth}{
    \tikz{
     \node[obs] (vx) {$\X_t$};%
     \node[latent,above=of vx] (vb) {$\B_t$}; 
     \node[latent,draw opacity=0, above right = of vb] (vtheta) {$\bm{\theta}_{it}$};
     \node[latent,draw opacity=0, below right = of vb, above right = of vx] (vdelta) {$\boldsymbol{\delta}_{\boldsymbol{\theta}_{it}}$};
     \node[obs, below=of vdelta, right=of vx](vw) {$W_{it}$};     
     \node[latent, below right =of vtheta, xshift=1cm] (vO) {$O_{it}$};
     \node[latent, below=of vO, right = of vw, pattern=north east lines, pattern color=gray!50] (vy) {$Y_{\mathit{it}}$};
     \node[obs, right = of vy](vz) {$Z_t$};
     
     \edge {vb, vdelta} {vx} ;
     \edge {vx} {vw};
     \path (vx) edge[bend right, ->] (vy);
     \edge {vtheta} {vdelta};
     \edge {vdelta} {vw};
     \path (vw) edge[bend right, ->] (vz);
     \path (vb) edge[bend right, in=150, out=30, <->] (vO); 
     \edge {vO} {vy};
     
     \plate [inner sep=.2cm, xshift=.1cm, yshift=-.2cm] {plate-e1}{(vtheta)(vdelta)(vw)(vO)(vy)}{};
     \node[above=of vO, xshift=-.4cm, yshift=-.4cm] (i) {$i=1,\ldots,n$}; 
     }}
     \caption{Competitive Selection}
     \label{fig:causal-multi-env}
  \end{subfigure}
     \caption{Causal graphs for our settings with selection variable $W_t$ (left) and multiple decision makers (right). The patterned nodes $Y_t$ and $Y_{it}$ represent the partial observability nature of these variables.} 
     \label{fig:causal-graphs}
\end{figure}

\section{Main Results}
\label{sec:main-results}
Our main results are based on the following two homogeneity assumptions on the strategic responses of agents.

\begin{enumerate}[align=left]
    \item[\textbf{H1.}] \emph{Homogeneous effort conversion}: for all $t\in[T]$, $\mathcal{E}_t=\homomatrix$ for some conversion matrix $\homomatrix$.
    \item[\textbf{H2.}] \emph{Homogeneous preference and compliance}: for each DM $i$ and for all $t\in[T]$, $\gamma_{it}=\gamma_{i}$ for some $\gamma_i \geq 0$ and $Z_t\indep\{\X_t,\B_t\}\mid\{W_{it}\}_{i=1}^n$.
\end{enumerate}
The former condition suggests that all agents exhibit the same strategic response regardless of their individual baselines, i.e., they only differ by their baselines $\B_t$, while the latter condition implies that all agents share the same preference over the $n$ DMs, and any two agents will demonstrate identical compliance behavior based on the given set of selection statuses $\{w_{it}\}_{i=1}^n$.
In the context of college admission, the common preference $\{\gamma_i\}_{i=1}^n$ may naturally align with the prestige of the colleges. Intuitively, these two assumptions suggest that while strategic responses may encompass both common and idiosyncratic elements, we solely concentrate on the common part, simplifying our theoretical analyses at the cost of potentially overlooking significant individual variations of agents' strategic behaviour. 

Our work thus concerns itself with a \emph{partially} heterogeneous setting. On the contrary, when completely heterogeneous agents are subjected to selection, many variables are rendered dependent; see, e.g., \cref{eq:conditional-objective}, making our theoretical analyses much more cumbersome. However, such homogeneity assumptions do not undermine the impact of selection that we discuss throughout this section since it is likely to persist in a more complex setting. This impact includes the trade-off between choosing capable agents and providing a maximal incentive, e.g., \Cref{col:max-improvement-s}, and the selection bias, e.g., \Cref{theorem:local-exo-s}.
To understand the impact of these two assumptions, we provide the sensitivity analyses in \apxAnalysesOfAssumptions. 
A relaxation of these assumptions will be considered in future work.

\subsection{Impact of Selection Procedure}
\label{subsec:single-selection}

To illustrate the impact of the selection procedure, we commence with the single DM setting, i.e. $n=1$.
For simplicity, we omit the subscript $i$ and assume that all agents comply with the decisions they receive. 
\Cref{fig:causal-single-env} shows the associated causal graph. 
The objective \eqref{eq:dm-objectives} for single DM then becomes
\begin{eqnarray}\label{eq:conditional-objective}
\mathbb{E}\left[Y_t\mid W_t=1;\boldsymbol{\theta}_t\right] &=& \mathbb{E}\left[\B^\top_t\bm{\theta}^*+O_t\mid W_t=1;\bm{\theta}_t\right] + \mathbb{E}\left[\left(\mathcal{E}_t\mathbf{a}_t\right)^\top\bm{\theta}^*\mid W_t=1;\bm{\theta}_t\right] \nonumber \\
&=& \text{cBP}(\bm{\theta}_t) + \text{cPI}(\bm{\theta}_t),
\end{eqnarray}%
where the first and second terms on the right-hand side are referred to as the \textit{conditional base performance} (cBP) and \textit{conditional performance improvement} (cPI), respectively. 
The former pertains to the agent outcome without strategic behavior, while the latter represents the improvement achieved through strategic behavior.
Both cBP and cPI are defined as expected values over the admitted agents, making them functions of the selection parameter $\bm{\theta}_t$. 
Additionally, the complexity of cBP and cPI relies on the chosen selection function $\bm{\delta}_{\bm{\theta}_t}$.
Our objective \eqref{eq:conditional-objective} differs from the marginal expected outcome commonly studied in prior work, where no selection occurs \citep{shavit2020causal, harris2022strategic, bechavod2022information}:
\begin{align}\label{eq:marginal-objective}
\mathbb{E}\left[Y_t;\ \boldsymbol{\theta}_t\right]
&= \mathbb{E}\big[\B^\top_t\bm{\theta}^*+O_t\big] + \mathbb{E}\big[\left(\mathcal{E}_t\mathbf{a}_t\right)^\top\bm{\theta}^*;\ \bm{\theta}_t\big].
\end{align}
We refer to the two terms on the right-hand side of \eqref{eq:marginal-objective} similarly as the \textit{marginal base performance} (mBP) and \textit{marginal performance improvement} (mPI). Observe that maximising \eqref{eq:marginal-objective} amounts to maximising only the mPI, whereas maximising our objective \eqref{eq:conditional-objective} might involve a trade-off between maximising cBP and cPI, as shown below. 

\textbf{Utility maximisation. } We further impose the following two assumptions, exclusively for utility maximisation:

\begin{enumerate}[align=left]
    \item[\textbf{S1}.]\textbf{Linear effect}: The selection yields a linear structure of cBP as follows: cBP$(\bm{\theta}_t) = \vecalpha^\top\bm{\theta}_t + \beta$ for some vector $\vecalpha\in\mathbb{R}^m$ and constant $\beta\in\mathbb{R}$;

    \item[\textbf{S2}.]\textbf{Bounded parameters}: For all $\bm{\theta}_t, \|\bm{\theta}_t\|_2\leq1$ \citep{shavit2020causal}. 
\end{enumerate}

On the one hand, Assumption S1 allows us to further simplify the analysis of the DM's behaviour and to further simplify the demonstration of the trade-off between choosing agents and incentivising them, which we discuss later. Even when S1 does not hold, this will only complicate the analysis without changing the implication resulted from \Cref{col:max-improvement-s} and \Cref{col:bounded-reduction-s}. On the other hand, as $\mathcal{Q}(\bm{\theta}_t)$ is not scale-invariant, we adopt Assumption S2, which was also used by previous work such as \citet{shavit2020causal} and \citet{bechavod2022information}. As a result, this allows us to restrict $\bm{\theta}_t$ to some arbitrarily small region and justifies a linear approximation to cBP$(\bm{\theta}_t)$. Nevertheless, we acknowledge the limitation of these assumptions and provide a more detailed discussion in \apxAnalysesOfAssumptions.

We denote the objective \eqref{eq:conditional-objective} by $\mathcal{Q}(\bm{\theta}_t)$ and expand it as
{\begin{align*}
\mathcal{Q}(\bm{\theta}_t) := \text{cBP}(\bm{\theta}_t) + \text{cPI}(\bm{\theta}_t) = \big(\vecalpha^\top\bm{\theta}_t + \beta\big) + \gamma\bm{\theta}_t^\top\homomatrix\homomatrix^\top\thetastar
\end{align*}}%
where we used the fact that $\mathbf{a}_t^\top=\gamma\bm{\theta}_t^\top\mathcal{E}_t$ and $\mathcal{E}_t=\homomatrix$ as a  result of Assumption (H1). 
Then, we formally state the optimal behaviour of the DM with the next theorem.

\begin{theorem}[Bounded optimum]\label{theorem:bounded-optim-s}
Suppose Assumptions (H1), (H2), (S1), and (S2) hold. Then, the optimal parameter vector for the DM can be expressed as 
$$\thetaao = \frac{\vecalpha+\gamma\homomatrix\homomatrix^\top\bm{\theta}^*}{\|\vecalpha+\gamma\homomatrix\homomatrix^\top\bm{\theta}^*\|_2}.$$
\end{theorem}

Since $\mathcal{Q}(\bm{\theta}_t)$ is linear in $\bm{\theta}_t$, the DM can obtain an un-normalised version of $\thetaao$ by regressing $\mathcal{Q}(\bm{\theta}_t)$ onto $\bm{\theta}_t$ using ordinary least squares (OLS) regression. As shown in \Cref{theorem:bounded-optim-s}, the optimal selection parameter $\thetaao$ is determined by the coefficients  $\vecalpha$ and $\homomatrix\homomatrix^\top\thetastar$ from cBP and cPI, respectively. 
Intuitively, this implies that an optimal selection rule might be a trade-off between selecting the best agents and incentivising agents to maximise their improvement.
\figTradeOff\ (in the supplementary material) illustrates when this trade-off happens and there exists no $\bm{\theta}_t$ for which both cBP$(\bm{\theta}_t)$ and cPI$(\bm{\theta}_t)$ are maximised simultaneously. 
The next corollary formalises this intuition.

\begin{corollary}[Maximum improvement]\label{col:max-improvement-s}
Suppose Assumptions (H1), (H2), (S1), (S2) hold, and $\gamma>0$. 
If $\vecalpha=(k-\gamma)\homomatrix\homomatrix^\top\bm{\theta}^*$ for some $k>0$, 
then the maximiser of $\mathcal{Q}(\bm{\theta}_t)$ is also the maximiser of cPI$(\bm{\theta}_t)$.
\end{corollary}

Generally speaking, the vector $\vecalpha$ represents the causal mechanism translating $\bm{\theta}_t$ into the base performance of the chosen agents, i.e., cBP$(\bm{\theta}_t)$, whereas the vector $\homomatrix\homomatrix^\top\thetastar$ denotes the causal mechanism translating $\bm{\theta}_t$ into the performance improvement of the selected agents, i.e., cPI$(\bm{\theta}_t)$. 
Hence, $\bm{\theta}_t$ serves not only as a selection parameter but also as the incentive for agents' improvement.
From \Cref{col:max-improvement-s}, when $k>\gamma$, the two aforementioned causal mechanisms align with each other, i.e., $\cos(\vecalpha, \homomatrix\homomatrix^\top\thetastar)=1$, then $\thetaao$ not only selects the best agents (i.e., in terms of cBP) but also is the incentive that maximises their improvement.

\textbf{Safeguarding the social welfare. } There is therefore a possibility that the deployed selection rule may result in undesirable societal outcomes.
For instance, this would involve rejecting agents who, with proper incentivisation, could have been chosen. Another example is when a decision rule selects the best agents but incidentally discourages them from further improvement, which corresponds to the case when $\cos(\thetaao,\homomatrix\homomatrix\thetastar)=-1$.
To prevent such situations, a benevolent regulator may opt to enforce a regulation such that a decision rule must result in $\cos(\thetaao,\homomatrix\homomatrix^\top\thetastar)>0$, thereby guaranteeing that the optimal parameters $\thetaao$ do not lead to a decline in the selected agents' outcome. 

In addition to the inherent trade-off induced by the selection process, \Cref{theorem:bounded-optim-s} also shows that $\thetaao$ differs from the true causal parameters $\thetastar$ in general. Relying on a selection criterion that uses the optimal parameters $\thetaao$ results in consistently inaccurate predictions of agents' outcomes. This unjustly reduces an agent's admission chance, compared to when the causal parameters were employed instead.
The following corollary then outlines conditions under which the reduction in an arbitrary agent's admission chance can be bounded when DM utilises $\thetaao$ as the selection parameter, instead of the causal parameter $\thetastar$.

\begin{corollary}[Bounded reduction]\label{col:bounded-reduction-s}
Suppose Assumptions (H1), (H2), (S1), (S2) hold, and the DM considers only two choices {$\bm{\theta}_t\in\{\bm{\theta}^*,\bm{\theta}^{\text{AO}}\}$}. Suppose further: 
\begin{enumerate}[label={(\arabic*)}]
    \item {$\|\bm{\theta}^*\|_2 \leq 1$};
    \item {$\vecalpha=(k-\gamma)\homomatrix\homomatrix^\top\bm{\theta}^*$} with {$k,\gamma>0$};
    \item {$\B_t\sim\mathcal{N}(0,\sigma^2I)$} where {$I$} is the identity matrix and {$\sigma\in\mathbb{R}^+$};
    \item The selection function, {$\bm{\delta}(\mathbf{x}; \bm{\theta}_t):=\widetilde{\bm{\delta}}(\hat{y}_t(\bm{\theta}_t))$}, is increasing w.r.t. {$\hat{y}_t$}. In addition, it is Lipschitz continuous, i.e. {$|\widetilde{\bm{\delta}}(\hat{y}) - \widetilde{\bm{\delta}}(\hat{y}^\prime)| \leq L|\hat{y}-\hat{y}^\prime|$} for {$L>0$}.
\end{enumerate}
Then, for any {$M>0$}:
\begin{equation*}
    p\left(\xi(\thetastar)-\xi(\thetaao)>M\right) \leq \Phi\left(\frac{-M/L-\lambda}{\sigma\|\thetaao-\thetastar\|_2}\right),    
\end{equation*}
where {$\xi(\bm{\theta}_t) := p(W_t=1|\B_t;\bm{\theta}_t)$} denotes the admission chance of the agent $t$, $\Phi$ denotes the CDF of $\mathcal{N}(0,1)$, and {$\lambda := \gamma\big((\thetaao)^\top\homomatrix\homomatrix^\top\thetaao - (\thetastar)^\top\homomatrix\homomatrix^\top\thetastar\big) \geq 0$}.
\end{corollary}

Specifically, this corollary rewrites the admission probability of an agent in terms of their baseline $B_t$ and denotes it with $\xi(\bm{\theta}_t)$. When the counterfactual quantity $\xi(\thetastar)-\xi(\thetaao)$ is positive for an agent, it implies that the admission chance of this agent will be reduced if the DM employs $\thetaao$ instead of $\thetastar$. As a result, this corollary gives us an upper bound for the probability of this reduction for an unknown agent.

\textbf{Causal parameters estimation. } As shown in \Cref{col:max-improvement-s} and \Cref{col:bounded-reduction-s}, it is necessary for the DM to know the incentivising causal mechanism, $\homomatrix\homomatrix^\top\thetastar$, in order to comply with the regulations, and for the regulator to know $\thetastar$ in order to verify the conditions of \Cref{col:bounded-reduction-s}.
Unfortunately, unbiased estimation of $\thetastar$ from observational data alone is impossible without imposing further assumptions on the data-generating process \citep{peters2017elements}.
In our case, unobserved common causes of the outcome $Y_t$ and covariates $\X_t$ create dependencies between $\X_t$ and $O_t$, rendering it impossible to estimate $\thetastar$ consistently via the OLS regression. 
To this end, \citet{harris2022strategic} proposes to view $\bm{\theta}_{t}$ as an instrumental variable (IV) and subsequently applies a two-stage least square (2SLS) regression \cite{cameron2005microeconometrics} to estimate $\thetastar$.
However, existing IV regression approaches are not suitable for our setting because the DM can only observe the outcomes of the selected agents, violating the unconfoundedness assumption of the IV; see \apxInvalidIV\ for the proof.

In what follows, we present an alternative approach to estimate $\thetastar$. This approach can be readily adapted to directly estimate $\homomatrix\homomatrix^\top\thetastar$. 
To that end, we first consider a ranking-based selection rule that is commonly deployed in practice.

\begin{definition}[Ranking selection]\label{def:ranking-selection} 
The DM selects an agent $t$ based on their relative ranking compared to other agents who are subject to the same selection parameters $\bm{\theta}_t$. Specifically, $$\bm{\delta}_{\bm{\theta}_t}(\mathbf{x}) = p\left(\X^\top_t\bm{\theta}_t\leq\mathbf{x}^\top\bm{\theta}_t\right) = \operatorname{CDF}_{\X^\top_t\bm{\theta}_t}\left(\mathbf{x}^\top\bm{\theta}_t\right).$$
\end{definition}

Based on this selection rule, the higher an agent's evaluation (relative to their peers) the more likely they will be selected. Note that in this work, we do not restrict the DM to the ranking selection rule for utility maximisation. This ranking selection rule is only provided so that the DM can retrieve the true causal parameters, which are useful for designing subsequent selection rules. The next theorem provides an unbiased estimate of the true causal parameter $\thetastar$ in our setting with a selection variable.

\begin{theorem}[Local exogeneity]\label{theorem:local-exo-s}
Under Assumptions (H1) \& (H2), if there exists a pair of rounds $t$ and $t^\prime$ such that $\bm{\theta}_t=k\bm{\theta}_{t^\prime}$ for some $k>0$, then we have: $$\mathbb{E}\left[Y_t \,\vert\, W_t=1\ ;\  \bm{\theta}_t\right] - \mathbb{E}\left[Y_{t^\prime} \,\vert\, W_{t^\prime}=1\ ;\ \bm{\theta}_{t^\prime}\right] 
= \Big(\mathbb{E}\left[\X_t \,\vert\, W_t=1\ ;\  \bm{\theta}_t\right]-\mathbb{E}\left[\X_{t^\prime} \,\vert\, W_{t^\prime}=1\ ;\  \bm{\theta}_{t^\prime}\right]\Big)^\top\bm{\theta}^*.$$
\end{theorem}

Intuitively, when all agents exert an equal amount of effort, ranking their covariates $\X_t$ is equivalent to ranking their baselines $\B_t$. Therefore, multiplying $\bm{\theta}_t$ by a positive scalar preserves the ranking. 
Consequently, we obtain a linear equation that contains no endogenous noise (from \Cref{theorem:local-exo-s}), allowing for unbiased estimation of the true causal parameters $\thetastar$. Specifically, if we refer to the left-hand side as $\Delta\bar{y}$ and the coefficient on the right-hand side as $\Delta\bar{\mathbf{x}}$, then $\thetastar$ can be estimated by regressing $\Delta\bar{y}$ onto $\Delta\bar{\mathbf{x}}$. 
We refer to this procedure as Mean-shift Linear Regression (MSLR) and discuss a sample algorithm in \Cref{sec:experiments}.

\subsection{Impact of Competitive Selection}
\label{subsec:competitive-selection}

When there are multiple DMs ($n\geq 2$), selecting an agent becomes competitive as their incentives affect the agent's covariates $\X_t$ simultaneously. Also, whether an agent complies with any DM is also influenced by other DMs' decisions. Consequently, additional assumptions are required to safeguard the agent's welfare as before. 

We assume each DM aims at unilaterally maximising the expected outcome of their own agents. We denote the objective of DM $i$ as $\max_{\bm{\theta}_{it}}\mathcal{Q}_i\left(\bm{\theta}_{t}^{all}\right)=\max_{\bm{\theta}_{it}} \mathbb{E}\left[Y_{it}\mid Z_t=i ; \bm{\theta}^{\text{all}}_t\right]$ and expand the expectation as
    \begin{align}\label{eq:cond-comp-obj}
    & \mathbb{E}\big[\B^\top_t\bm{\theta}^*_i + O_{it} \mid Z_t=i ; \bm{\theta}^{\text{all}}_t\big] + \big(\sum^n_{j=1}\gamma_j\bm{\theta}_{jt}\big)^\top\homomatrix\homomatrix^\top\bm{\theta}_i^*
    \end{align}%
where the first and second terms are denoted similarly as $\text{cBP}_i(\bm{\theta}_t^\text{all})$ and $\text{cPI}_i(\bm{\theta}_t^\text{all})$, respectively, and $\mathbf{a}_t^\top=(\sum_{j=1}^n\gamma_{jt}\bm{\theta}_{jt})^\top\homomatrix$ is again the agents' optimal strategic action.
We highlight that the objective \eqref{eq:cond-comp-obj} depends not only on $\bm{\theta}_{it}$ but also on $\bm{\theta}_t^{-i}$ due to the interaction between DMs via competitive selection. 
As a result, $\mathcal{Q}_i(\{\bm{\theta}_{it},\bm{\theta}_t^{-i}\})$ can be seen as a family of objective functions parameterised by $\bm{\theta}_t^{-i}$. 
When DM $i$ is an expected-utility maximiser, we would maximise the expectation of $\mathcal{Q}_i(\{\bm{\theta}_{it},\bm{\theta}_t^{-i}\})$ to marginalise out the effect of $\bm{\theta}_t^{-i}$. However, this requires knowledge on the conditional distribution $p(\bm{\theta}_t^{-i}| Z_t=i,\bm{\theta}_{it})$ which the DM $i$ does not have. To tackle this challenge, we consider the worst-case scenario in which all rival DMs cooperate to minimise the objective and study how DM $i$ can in response maximise this worst-case objective function. We show in \apxMaximin\ that our solution, specifically in this case, is also a maximin strategy of the DM $i$.

\textbf{Utility maximisation. } The objective~\eqref{eq:cond-comp-obj} is difficult to optimise as it depends not only on the choice of selection rules but also on the behaviour of other DM's objective. 
To simplify the analysis, we rely on the following assumptions, exclusively for utility maximisation:
\begin{enumerate}[align=left]
    \item[\textbf{M1}.] \textbf{Partially additive interaction between DMs}: For an arbitrary DM $i$, their cBP$_i$ can be decomposed as cBP$_i(\{\bm{\theta}_{it},\bm{\theta}_t^{-i}\})\
    =\ g_{i}(\bm{\theta}_{it}) + h_i(\bm{\theta}_t^{-i}) + c_i$ for some function $g_i, h_i$ and constant $c_i$.

    \item[\textbf{M2}.] \textbf{Linear self-effect}: The contribution of DM $i$ to cBP$_i$ admits a linear structure, i.e. $g_i(\bm{\theta}_{it}) = \vecalpha_{i}^\top\bm{\theta}_{it} + \beta_i$ for some vector $\vecalpha_i\in\mathbb{R}^m$ and constant $\beta_i\in\mathbb{R}$;

    \item[\textbf{M3}.] \textbf{Bounded parameters}: For all $\bm{\theta}_{it}, \|\bm{\theta}_{it}\|_2\leq1$.

\end{enumerate}

Assumptions (M2) \& (M3) are extensions of (S1) and (S2), whereas (M1) is an additional assumption.

\begin{proposition}[Dominant strategy]\label{prop:dom-strat}
Suppose Assumptions (H1), (H2), and (M1) hold for DM $i$. Then, $\arg\max_{\bm{\theta}_{\mathit{it}}}\mathcal{Q}_i(\{\bm{\theta}_{\mathit{it}},\bm{\theta}_{\bullet}^{-i}\})=\arg\max_{\bm{\theta}_{\mathit{it}}}\mathcal{Q}_i(\{\bm{\theta}_{\mathit{it}},\bm{\theta}_{\diamond}^{-i}\})$ for any 
pair of distinct values $\bm{\theta}_{\bullet}^{-i}$ and $\bm{\theta}_{\diamond}^{-i}$ of $\bm{\theta}_t^{-i}$. 
\end{proposition}

This proposition shows that the monotonicity of the objective $\mathcal{Q}_i$ remains unaffected regardless of other values $\boldsymbol{\theta}_{jt}$ released by the DMs $j$ where $j\neq i$. 
Based on Assumption (M1) and this result, any DM $i$ is provided with a condition to maximise all the objective functions within the family $\mathcal{Q}_i(\{\bm{\theta}_{it},\bm{\theta}_t^{-i}\})$ simultaneously.

\begin{theorem}[Bounded optimum, extended]\label{theorem:bounded-optim-m}
Suppose that (H1), (H2), and (M1)-(M3) hold. Then, the optimal parameter vector for any DM $i$ takes the form $$\bm{\theta}^{\text{AO}}_i= \frac{\vecalpha_i+\gamma_i\homomatrix\homomatrix^\top\bm{\theta}^*_i}{\|\vecalpha_i+\gamma_i\homomatrix\homomatrix^\top\bm{\theta}^*_i\|_2}.$$
\end{theorem}

As each objective function $\mathcal{Q}_i(\{\bm{\theta}_{it},\bm{\theta}_\diamond^{-i}\})$, conditioned on some arbitrary $\bm{\theta}_t^{-i}:=\bm{\theta}_\diamond^{-i}$, is a linear function of $\bm{\theta}_{it}$, the DM $i$ can obtain an un-normalised version of $\bm{\theta}_i^\text{AO}$ by regressing $\mathcal{Q}_i(\{\bm{\theta}_{it},\bm{\theta}_\diamond^{-i}\})$ onto $\bm{\theta}_{it}$ using the OLS regression.

\textbf{Safeguarding the social welfare. } Like the previous setting, we provide below an extension of \Cref{col:max-improvement-s} on maximum agents' improvement, with the difference being that $\thetaao_i$ is also a dominant strategy for cPI$(\{\bm{\theta}_{it},\bm{\theta}_t^{-i}\})$.

\begin{corollary}[Maximum improvement, extended]\label{col:max-improvement-m}
Suppose Assumptions (H1), (H2), and (M1)-(M3) hold for an arbitrary DM $i$, and that $\gamma_i>0$. 
If $\vecalpha_i=(k_i-\gamma_i)\homomatrix\homomatrix^\top\thetastar_i$ for some $k_i>0$, then $\thetaao_i$ maximises both $\mathcal{Q}_i(\{\bm{\theta}_{it},\bm{\theta}_t^{-i}\})$ and cPI$_i(\{\bm{\theta}_{it},\bm{\theta}_t^{-i}\})$, regardless of $\bm{\theta}_t^{-i}$.
\end{corollary}

As a result, if the interaction between DMs and agents exhibits additive structures, regulations can be solely imposed on DM $i$ to ensure improved average outcome of the agents who are selected by (and comply with) DM $i$. 
Next, we extend \Cref{col:bounded-reduction-s} regarding agents' admission chance for the environment $i$.

\begin{corollary}[Bounded reduction, extended]\label{col:bounded-reduction-m}
Suppose assumptions (H1), (H2), (M1)-(M3) hold for all DMs and each DM considers only two choices $\bm{\theta}_{jt}\in\{\bm{\theta}_j^*,\bm{\theta}_j^{\text{AO}}\}$ for $j\in[n]$. Further, let $i$ be an arbitrary DM and suppose:
\begin{enumerate}[label={(\arabic*)}]
    \item $\|\bm{\theta}_i^*\|_2 \leq 1$; 
    \item $\vecalpha_j=(k_j-\gamma_j)\homomatrix\homomatrix^\top\bm{\theta}_j^*$ with $k_j,\gamma_j>0$ for $j\in[n]$;
    \item $(\bm{\theta}_{jt})^\top\homomatrix\homomatrix^\top(\thetaao_i-\thetastar_i) \geq 0$ for $j\neq i$;
    \item $\B_t\sim\mathcal{N}(0,\sigma^2I)$ where $I$ is the identity matrix and $\sigma\in\mathbb{R}^+$; 
    \item The selection function, $\bm{\delta}_i(\mathbf{x}; \bm{\theta}_{it}):=\widetilde{\bm{\delta}}_i(\hat{y}_{it}(\bm{\theta}_t^\text{all}))$, is increasing w.r.t. $\hat{y}_{it}$. In addition, it is Lipschitz continuous, i.e., $ |\widetilde{\bm{\delta}}_i(\hat{y}_{it}) - \widetilde{\bm{\delta}}_i(\hat{y}_{it}^\prime)| \leq L|\hat{y}_{it}-\hat{y}_{it}^\prime|$ for $L>0$.
\end{enumerate}
Then, for any $M>0$:
\begin{equation*}
p\big(\xi_i(\bm{\theta}^\text{all}_{\bullet})-\xi_i(\bm{\theta}^\text{all}_{\diamond})>M\big) \leq \Phi\left(\frac{-M/L-\lambda}{\sigma\|\thetaao_i-\thetastar_i\|_2}\right),
\end{equation*}
where $\bm{\theta}^\text{all}_{\bullet} = \{\bm{\theta}_t^{-i},\thetastar_i\}$, $\bm{\theta}^\text{all}_{\diamond} = \{\bm{\theta}_t^{-i},\thetaao_i\}$, and $\xi_j(\bm{\theta}_t^\text{all}) := p(W_{jt}=1|\B_t;\bm{\theta}_t^\text{all})$ denotes the admission chance of the agent $t$ from the DM $j$, given released decision parameters $\bm{\theta}_t^\text{all}$. $\Phi$ denotes the CDF of $\mathcal{N}(0,1)$ and $\lambda := \big(\mathbf{a}_t(\bm{\theta}_{\diamond}^\text{all})\big)^\top\homomatrix^\top\thetaao_i - \big(\mathbf{a}_t(\bm{\theta}_{\bullet}^\text{all})\big)^\top\homomatrix^\top\thetastar_i \geq 0$.
\end{corollary}

As demonstrated in the corollary, this extension necessitates additional conditions compared to \Cref{col:bounded-reduction-s} to ensure the protection of agents' chances of being selected, thus effectively highlighting the impact of competitive selection. 
Firstly, as shown in \Cref{sec:prob-formulation}, the agent's best response is influenced by all DMs in which the agent has an interest, i.e., $\gamma_j>0$.
This dynamic creates competition among DMs to incentivise agents effectively. 
Secondly, the compliance status $z_t$ of an agent depends on its selection statuses, $\{w_{\mathit{jt}}\}_{j=1}^n$, which in turn are affected by the selection rules $\bm{\delta}_{\bm{\theta}_{\mathit{jt}}}$ for $j\in[n]$, resulting in another competition in evaluating and selecting agents.

The third condition in \Cref{col:bounded-reduction-m} is of particular interest. 
Recall that under this corollary, any $\thetaao_j$ is simply a normalised version of $\homomatrix\homomatrix^\top\thetastar_j$. Consequently, this condition implies that when agents prefer only DMs whose environments are sufficiently similar to each other (reflected via the $\bm{\thetastar}_j$ and $\bm{\thetastar}_i$), then the benevolent regulator can enforce the regulation as outlined in \Cref{col:max-improvement-m} for all DMs to protect the agents from unjust reductions in admission chances.
We provide a more detailed discussion on this third condition in \apxBoundedReductionMulti. We delay the discussion on an agent's admission chance into other environments $j\neq i$ in the appendix.

\textbf{Causal parameters estimation. } Similar to the previous setting, all DMs must have access to the causal mechanism $\homomatrix\homomatrix^\top\thetastar_i$ to be able to comply with the regulation in \Cref{col:max-improvement-m} and the regulator must know $\thetastar_i$ to let agents know which environments are sufficiently similar (\Cref{col:bounded-reduction-m}). 
Unfortunately, DMs encounter additional challenges estimating causal parameters  $\{\thetastar_i\}_{i=1}^n$ under competitive selection. 
Specifically, they cannot correct the selection bias alone due to the interference with other DMs, which we demonstrate empirically in the following section. 
To address this, we propose a cooperative protocol for the regulator. 
This ensures unbiased estimation of causal parameters for all DMs.

\begin{definition}[Cooperative protocol]\label{def:protocol}
Let $[n]$ be the set of all $n$ DMs. If for any two arbitrary rounds $t$ and $t^\prime$, a non-empty subset of DMs, $S\subseteq [n]$, employs the ranking selection rule (\Cref{def:ranking-selection}) and their respective decision parameters satisfy $\bm{\theta}_{it}=k_i\bm{\theta}_{it^\prime}$ for some $k_i>0$ for all $i\in S$, then we say that $S$ follows the cooperative protocol in these two rounds.
\end{definition}

This cooperative protocol extends the condition in \Cref{theorem:local-exo-s} and suggests that DMs should synchronise the releases of their positively scaled parameters. When all DMs follow the cooperative protocol for multiple pairs of rounds, they have a way to obtain unbiased estimates of the causal parameters $\thetastar_i$ as shown in the next theorem.

\begin{theorem}[Local exogeneity, extended]\label{theorem:local-exo-m}
Under Assumptions (H1) and (H2). If all DMs follow the cooperative protocol (i.e., $\exists S: S=[n]$ in \Cref{def:protocol}) in two rounds $t$ and $t^\prime$, then 
 $$\mathbb{E}\left[Y_{it} \,\vert\, Z_{t}=i\ ;\  \bm{\theta}^{\text{all}}_{t}\right] - \mathbb{E}\left[Y_{it^\prime} \,\vert\, Z_{t^\prime}=i\ ;\ \bm{\theta}^{\text{all}}_{t^\prime}\right] = \big(\mathbb{E}\left[\X_{t} \,\vert\, Z_{t}=i\ ;\  \bm{\theta}^{\text{all}}_{t}\right]-\mathbb{E}\left[\X_{t^\prime} \,\vert\, Z_{t^\prime}=i\ ;\  \bm{\theta}^{\text{all}}_{t^\prime}\right]\big)^\top\bm{\theta}^*_i.$$
\end{theorem}

Recall that in the previous setting, the ranking of agents can be preserved by scaling the selection parameters with a positive scalar. With the cooperative protocol and Assumption (H2), we can now also preserve the enrollment distribution of agents. We can then deploy the same MSLR procedure from the single DM settings to retrieve the causal parameters. Further details are discussed in \Cref{sec:experiments}.

\begin{algorithm}
\caption{Mean-shift Linear Regression (MSLR)}
\label{algo:main}
\textbf{Require:} a subset of $n_s$ decision makers out of all $n$ decision makers, where $1\leq n_s\leq n$. These decision makers use ranking selection (\Cref{def:ranking-selection}).\\
\textbf{Parameters:} number of rounds $T$, block's length $\eta$.

\begin{algorithmic}[1]
\STATE $D_i \gets \{\}$ for $i=1,\ldots,n_s$
\FOR{$t\in\{1,\ldots,T\}$}
    \STATE blockindex $\gets \lfloor t/(\eta+1) \rfloor$
    \IF{blockindex $\% 2 = 0$}
    \STATE $\bm{\theta}_{it}\sim p(\bm{\theta}_{it})$ for $i=1,\ldots,n_s$
    \ELSE
    \STATE $t^\prime \gets t-\eta$
    \FOR{$i\in\{1,\ldots,n_s\}$}
    \STATE $\bm{\theta}_{it} = k_{it}\bm{\theta}_{it^\prime}$ with $k_{it}>0$
    \STATE $\Delta\bar{y}_i \gets (\bar{y}_{it}\mid z_t=i)-(\bar{y}_{it^\prime}\mid z_{t^\prime}=i)$
    \STATE $\Delta\bar{\mathbf{x}}_i \gets (\bar{\mathbf{x}}\mid z_t=i)-(\bar{\mathbf{x}}\mid z_{t^\prime}=i)$
    \STATE $D_i \gets D_i \cup \{\Delta\bar{y}_i, \Delta\bar{\mathbf{x}}_i\}$
    \ENDFOR
    \ENDIF
\ENDFOR
\FOR{$i\in\{1,\ldots,n_s\}$}
    \STATE $\thetastarhat_i \gets$ Regress $\Delta\bar{Y}_i$ onto $\Delta\bar{\mathbf{X}}_i$ with OLS and the data set $D_i$
\ENDFOR
\end{algorithmic}
\end{algorithm}

\section{Experiments}
\label{sec:experiments}

We complement our theoretical results with simulation studies. Starting with the single DM setting, our experiments first compare the optimal decision parameters $\thetaao$ and the causal parameter $\thetastar$ in terms of utility maximisation, and then we demonstrate that our algorithms estimate $\thetastar$ consistently. We then generalise the experiments to multiple DMs. Further experimental details are included in \apxExperiments. The code to reproduce our experiments is publicly available.\footnote{\url{https://github.com/muandet-lab/csl-with-selection}}

\textbf{Experimental setup. } Following \citet{harris2022strategic}, we generate a synthetic college admission dataset. In particular, covariates $\X_t = (X_t^\text{SAT}, X_t^\text{HS GPA})^\top$ represent SAT score and high school GPA of the student arriving at round $t$, while $Y_{it}$ represents the college GPA after enrolling in college $i$. A confounding factor is simulated to indicate the private type of a student's background: \textit{disadvantaged} and \textit{advantaged}. The distribution of the \textit{disadvantaged} students' baseline $\B_t$ has a lower mean than that of their \textit{advantaged} counterparts and the same applies for the distribution of noise $O_{it}$. After $\mathbf{b}_t$ is simulated and all colleges publicise their parameters $\{\bm{\theta}_{it}\}_{i=1}^n$, we compute $\mathbf{x}_t = \mathbf{b}_t + \homomatrix\mathbf{a}_t$ and $\hat{y}_{it}=\mathbf{x}_t^\top\bm{\theta}_{it}$ for $i\in[n]$. Unlike \citet{harris2022strategic}, each DM $i$ now assigns an admission status $w_{it}\in\{0,1\}$ to the student at round $t$ using a variant of the ranking selection rule. Precisely, the student is admitted into college $i$ if their prediction $\hat{y}_{it}$ lies within the top $\rho$-percentile of all applicants where $\rho\in[0,1]$ and we set $\rho=0.5$. Further discussion of this variant of ranking selection is included in \apxExpSetup. 
As ranking selection (\Cref{def:ranking-selection}) requires access to the distribution $p(\X_t^\top\bm{\theta}_{it})$, we estimate it by simulating $1000$ students in each round.\footnote{Having multiple students per round is equivalent to having each student arrive at different rounds subject to the same  $\bm{\theta}^\text{all}$.} Afterwards, the compliance $z_t \in [n]$ is computed to indicate the college in which this student enrols, based on the admission statuses $\{w_{it}\}_{i=1}^n$. Finally, for students enrolled in college $i$ at round $t$, i.e., $z_t = i$, we compute the target college GPA $y_{it} = \mathbf{x}_t^\top\thetastar_i + o_{it}$. The true causal parameters $\bm{\theta}^*_i = (\theta_i^{*, \text{SAT}}, \theta_i^{*, \text{HS GPA}})^\top$ are distributed as normal distribution around $\bm{\theta}^*_i = (0, 0.5)^\top$, which was inferred from a real world dataset by \citet{harris2022strategic}. 

\textbf{Additional details for MSLR. }
Because there are infinitely many ways to carry out the releases of $\bm{\theta}_t$ as required by \Cref{theorem:local-exo-s} and there are infinitely many ways for multiple DMs to synchronise their releases of $\bm{\theta}_{it}$ as required by \Cref{def:protocol}, we provide only an instantiation of the MSLR procedure via \Cref{algo:main} that we use in our experiments.
We use the word \textit{coalition} to refer to the subset of $n_s$ DMs who perform this algorithm together. Line 9 refers to the cooperative protocol (\Cref{def:protocol}) and line 10 to 12 refer to the extended theorem on local exogenity (\Cref{theorem:local-exo-m}). The branching in line 4 (and in line 6) checks whether the current round $t$ is of the type $t_\bullet$ or $t_\diamond$ which we discuss next.
Recall that according to \Cref{def:protocol}, DMs $i$ and $j$ are cooperative if they deploy linearly dependent parameter vectors $\{\bm{\theta}_{it_\bullet}, \bm{\theta}_{it_\diamond}\}$ in the same pair of two arbitrary rounds $t_\bullet$ and $t_\diamond$. To easily simulate the cooperative and non-cooperative aspects of DMs in our experiments, we control the interval for deploying dependent vectors with the integer constants $\eta_i\in\mathbb{N}^+$ where $t_\bullet+\eta_i=t_\diamond$. Each batch of such linearly dependent vectors gives us a linear equation as shown in \Cref{theorem:local-exo-m} and we want to have multiple distinct batches with sufficient span in $\mathbb{R}^m$ so that $\thetastar_i$ is solvable. Because $\eta_i$ creates a gap between $t_\bullet$ and $t_\diamond$ of the same batch, distinct batches are generated in an interleaved manner using the following formula:
\begin{equation*}
    t_\bullet = k + \left(\left\lfloor \frac{k-1}{\eta_i}\right\rfloor \times \eta_i\right), \quad
    t_\diamond = t_\bullet + \eta_i,
\end{equation*}
where $k \in \{1, 2, \ldots\}$ denotes the $k$-th batch to which $\{\bm{\theta}_{it_\bullet},\bm{\theta}_{it_\diamond}\}$ belong. Finally, we say that a set of DMs deploys the parameter vectors \textit{synchronously} if they deploy the linearly dependent vector at the same frequency (i.e., $\forall i, \eta_i=\eta$ for some constant $\eta$), otherwise, we say their deployments are \textit{asynchronous}.

\begin{table}
\centering
\caption{[Higher is better] Utility $\mathcal{Q}(\bm{\theta}_t)$ ($\pm$ standard error) of a DM for various values of the parameter $\bm{\theta}_t$.}
\begin{tabular}{lccc}
& $\thetaaohat$  & $\hat{\bm{\theta}}_{\text{OLS}}$ & $\thetastarhat$\\ 
\midrule
\multicolumn{1}{l|}{$\mathcal{Q}(\bm{\theta}_t)$} & ${\bf 2.530}$ {\small$\pm0.006$} & $2.511$ {\small$\pm0.006$} & $2.511$ {\small$\pm0.006$} \\ 
\bottomrule
\end{tabular}
\label{table:utility-single-env}
\end{table}
\begin{table}
\centering
\caption{[Higher is better] Utilities $\mathcal{Q}_1(\bm{\theta}_{1t}, \bm{\theta}_{2t})$ ($\pm$ standard error) of the first DM for various values of $\{\bm{\theta}_{1t}, \bm{\theta}_{2t}\}$. }
\begin{tabular}{lccc}
& $\thetaaohat_1$ & $\thetaolshat_1$ & $\thetastarhat_1$  \\ \midrule
\multicolumn{1}{l|}{$\thetaaohat_2$} & $\bm{2.529}$  {\small $\pm0.028$} & $2.507$ {\small$\pm0.029$} & $2.506$ {\small$\pm0.029$} \\ 
\multicolumn{1}{l|}{$\thetaolshat_2$} & $\bm{2.561}$ {\small $\pm0.028$}  & $2.546$ {\small$\pm0.029$} & $2.545$ {\small$\pm0.029$} \\ 
\multicolumn{1}{l|}{$\thetastarhat_2$} & $\bm{2.560}$ {\small$\pm0.029$} & $2.546$ {\small$\pm0.029$} & $2.544$ {\small$\pm0.029$} \\ \bottomrule
\end{tabular}
\label{table:utility-multi-env}
\end{table}

\textbf{Impact of selection procedure ($n=1$). } We first demonstrate our estimated $\thetaao$ in fact results in higher utility than other plausible selection parameters such as $\thetastarhat$ and $\thetaolshat$, echoing the theoretical analysis from \Cref{theorem:bounded-optim-s}. We regress $\mathcal{Q}(\bm{\theta}_{t})$ onto $\bm{\theta}_{t}$ to estimate $\thetaao$~(see \Cref{subsec:single-selection}), and utilise our MSLR algorithm to estimate $\thetastar$, whereas $\thetaolshat$ is obtained from performing ordinary regression with $Y_{t}|W_t=1$ and $\X_t|W_t=1$.
Conforming to Assumption (S2), we use $\|\thetastarhat\|_2$ as the threshold and scale $\thetaaohat$, such that $\|\thetaaohat\|_2=\|\thetastarhat\|_2$ to ensure a fair comparison. On the other hand, if $\thetaolshat$ has a larger magnitude than the threshold, we scale it down accordingly (see \apxBoundingTheta\ for the detailed explanation).
\Cref{table:utility-single-env} reports their utility values $\mathcal{Q}(\bm{\theta}_t)$. 
We can see that $\thetaaohat$ induces the highest utility compared to other plausible options of $\bm{\theta}_t$. 
To demonstrate the impact of selection on estimating $\thetastar$, which is needed for the DM to comply with the regulation (\Cref{col:max-improvement-s}), we compare the MSLR algorithm (cf. \Cref{theorem:local-exo-s}) with that of \citet{harris2022strategic}, i.e., 2SLS. 
\Cref{fig:results1} shows estimation errors as the number of rounds increases.
Unlike the OLS and 2SLS estimates, our estimate of $\thetastar$ is asymptotically unbiased.

\begin{figure*}[t!]
     \centering
     \begin{subfigure}[t]{0.75\textwidth}
        \centering
        \includegraphics[width=\textwidth]{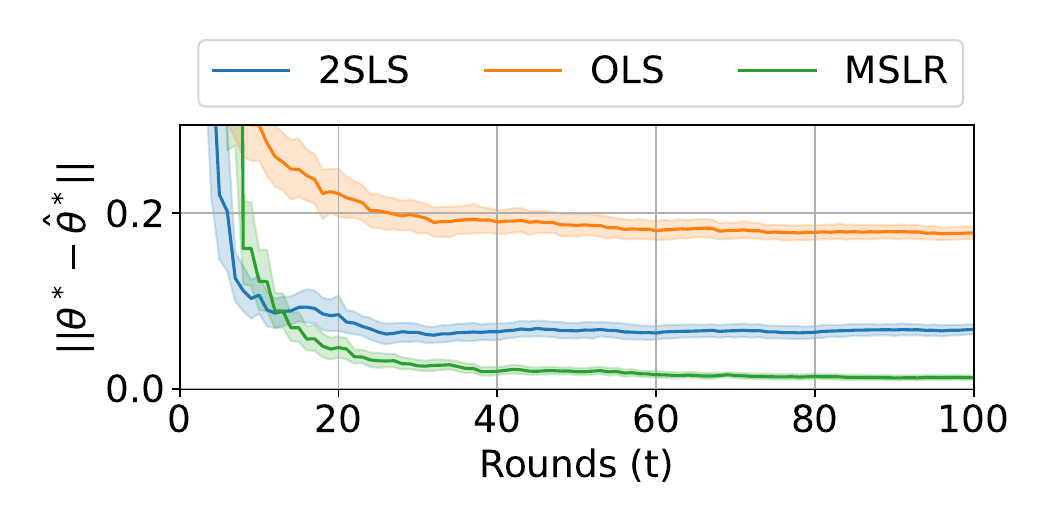}
        \caption{Estimation errors of causal parameter estimate $\thetastarhat$ as a function of rounds.}
        \label{fig:results1}
      \end{subfigure}\\
     \begin{subfigure}[t]{0.75\textwidth}
         \centering
         \includegraphics[width=\textwidth]{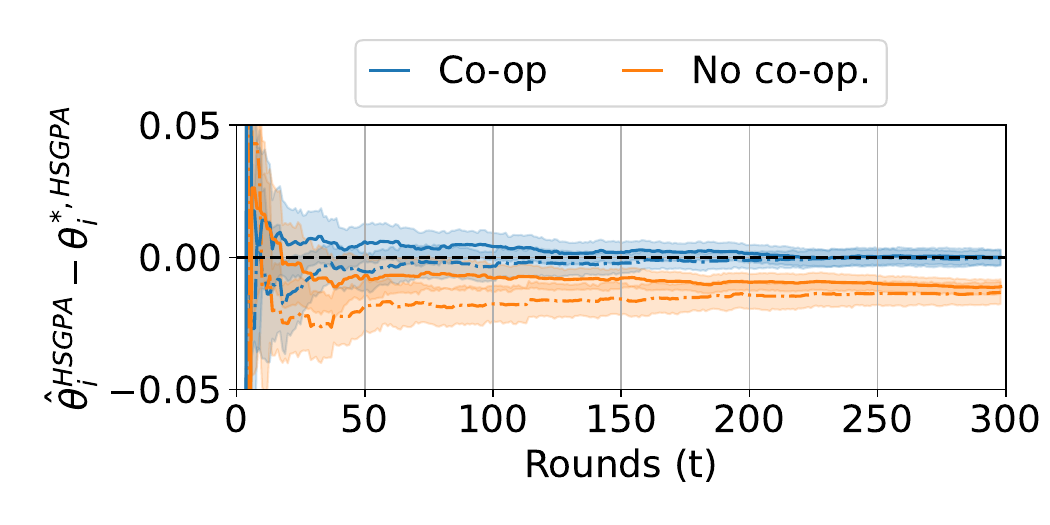}
         \caption{Bias in the estimated causal effect of $X_{t}^{\text{HS GPA}}$ on $Y_{it}$ for 2 DMs, each depicted by a different line style.}
         \label{fig:results2}
     \end{subfigure} \\
     \begin{subfigure}[t]{0.75\textwidth}
         \centering
         \includegraphics[width=\textwidth]{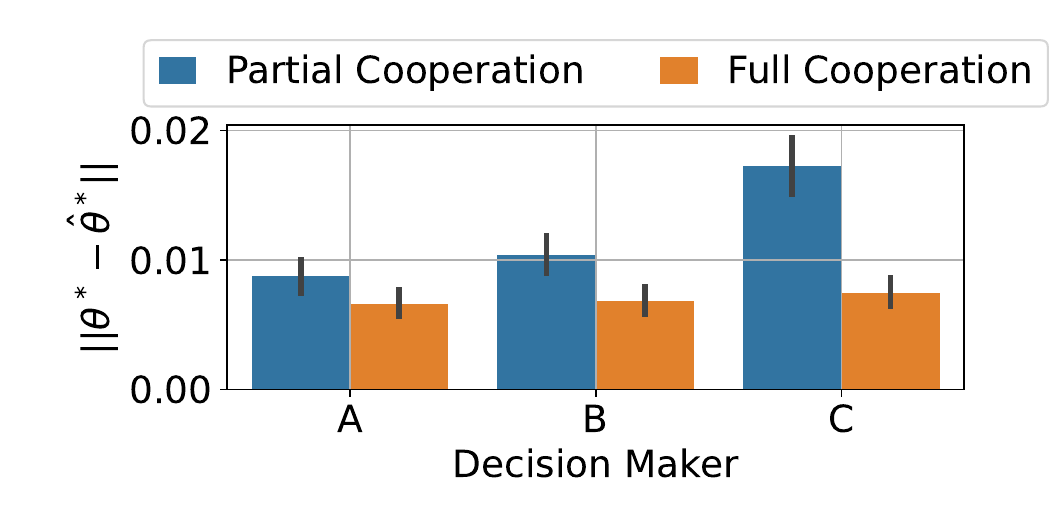}
         \caption{Estimation error of each DM when two of them cooperate ($A+B$) and when all of them cooperate ($A+B+C$).}
         \label{fig:results3}
     \end{subfigure}
     \caption{Estimation of the causal parameters $\bm{\theta}^*$ under various scenarios. Error bars show 95\% confidence interval.} 
      \label{fig:results}
\end{figure*}


\textbf{Impact of competitive selection ($n\geq2$). } Next, we show that $\thetaaohat_1$ induces the optimal utility $\mathcal{Q}_1(\bm{\theta}_{1t}, \bm{\theta}_{2t})$ for the first DM as a dominating strategy. Analogous to the previous experiment, \Cref{table:utility-multi-env} shows that our estimate $\thetaaohat_1$ induces the highest utility $\mathcal{Q}_1$ for the first DM, regardless of $\bm{\theta}_{2t}$ deployed by the second DM. We normalise the parameter similarly as before and use $\|\thetastarhat_1\|_2$ and $\|\thetastarhat_2\|_2$ as thresholds. 

We now demonstrate the impact of competitive selection on the estimation of causal parameter $\thetastar_i$ for $i\in[n]$, which are needed for DMs to comply with our regulations.
By \Cref{theorem:local-exo-m}, they must follow the cooperative protocol (\Cref{def:protocol}) by deploying linearly dependent parameter vectors $\bm{\theta}_{it} = k_i \bm{\theta}_{it'}$ in the same pair of rounds $t$ and $t'$. 
To this end, we test whether DMs can estimate $\thetastar$ when they deploy linearly dependent vectors (a) synchronously, as required by the protocol (i.e., cooperation), and (b) asynchronously (i.e., no cooperation).
\Cref{fig:results2} shows that cooperation enables all DMs to obtain unbiased estimates of $\theta^{*, \text{HS GPA}}$, the ground-truth causal effect of the high-school GPA covariate. 
We provide the results for the other covariate in \apxFurtherExp. 
Lastly, we demonstrate that following the cooperative protocol is of mutual benefit to all DMs for obtaining accurate estimates of $\thetastar_i$. To this end, we generate the data with $n=3$, for two scenarios: a group of two DMs ($A$ and $B$) deploys linearly dependent parameter vectors synchronously, while the remaining DM ($C$) deploys its respective linearly dependent vector (i) asynchronously (i.e., leading to partial cooperation between DMs), and (ii) synchronously (i.e., full cooperation). We use the converged estimates (i.e., after $T=100$ rounds) of causal parameters under both scenarios to demonstrate that, in terms of accuracy, not only does the DM $C$ gain substantially by joining the coalition, but it also benefits the current members of the coalition; see \Cref{fig:results3}.

\section{Conclusion}
\label{sec:conclusion}

To conclude, we study the problem of causal strategic learning under competitive selection by multiple decision makers. We show that in this setting, optimal selection rules require a trade-off between choosing the best agents and motivating their improvement. In addition, these rules may unjustly reduce the admission chances of agents due to reliance on non-causal predictions. To address these issues, we propose conditions for a benevolent regulator to impose on decision makers, allowing them to recover true causal parameters from observational data and ensure optimal incentives for agents' improvement without excessively reducing their admission chances, thus safeguarding agents' welfare. 

Our results rest on assumptions like homogeneous strategic behavior and linearity in agent models. Although these assumptions undoubtedly limit the applicability of our methods, they do not undermine the implication of our work. Intuitively, this inherent trade-off emerges because a DM has only one degree of freedom in designing the selection rule that may result in two distinct effects. Consequently, selecting the best candidates (private reward) and incentivising their improvements (social return) can indeed differ; and when they do, the benevolent regulator, e.g., governments, is needed to align the two. Our finding reinforces causal identification as an essential instrument to achieve this. Future studies could delve into non-linear agent models, fully heterogeneous setting, or scenarios in which certain decision makers cooperate strategically.

\section*{Acknowledgments} 
We thank Jake Fawkes and Nathan Kallus for a fruitful discussion and detailed feedback. We also thank David Kaltenpoth, Jilles Vreeken, and Xiao Zhang for insightful questions and feedback on the preliminary version of this work which was presented at the CISPA ML Day.

\bibliographystyle{unsrtnat}
\bibliography{references}  






\clearpage
\newpage
\appendix


\section{Proofs of the Main Results}

This section contains the proofs of our main results. 
For proofs of \Cref{theorem:bounded-optim-s} and \Cref{col:max-improvement-s}, we refer the readers to the corresponding proofs in the multiple decision makers setting; see \Cref{apx:bounded-opt-ext} and \Cref{apx:max-imp-ext}.

\subsection{Agents' Best Response}\label{apx:best-response}

The objective of each student $t$ is to maximise his/her chance via the predicted performance, i.e., 
\begin{align*}
 \max_{\mathbf{a}_t}\ \sum^n_{i=1}\gamma_{it}\widehat{y}_{it} - \frac{1}{2}\|\mathbf{a}_t\|^2_2
=\ &\max_{\mathbf{a}_t}\ \sum^n_{i=1}\gamma_{it}\left(\mathbf{b}_t+\mathcal{E}_t\mathbf{a}_t\right)^\top\bm{\theta}_{it} - \frac{1}{2}\|\mathbf{a}_t\|^2_2
\\
=\ &\max_{\mathbf{a}_t}\ \sum^n_{i=1}\gamma_{it}\mathbf{b}^\top_t\boldsymbol{\theta}_{it} + \sum^n_{i=1}\gamma_{it}\mathbf{a}_t^\top\mathcal{E}_t^\top\bm{\theta}_{it} - \frac{1}{2}\|\mathbf{a}_t\|^2_2
\\
=\ &\max_{\mathbf{a}_t}\ \sum^n_{i=1}\gamma_{it}\mathbf{b}^\top_t\boldsymbol{\theta}_{it} + \mathbf{a}^\top_t\mathcal{E}_t^\top\left(\sum^n_{i=1}\gamma_{it}\bm{\theta}_{it}\right) - \frac{1}{2}\|\mathbf{a}_t\|^2_2,
\end{align*}
where the first equality follows from the fact that $\hat{y}_{it}$ is a linear function of $\mathbf{x}_t$, which in turn is a linear function of $\mathbf{a}_t$. Next, using the expansion in the last line and setting its gradient to zero to solve for $\mathbf{a}_t$ yield the optimal action $\mathbf{a}_t=\mathcal{E}_t^\top\left(\sum^n_{i=1}\gamma_{it}\boldsymbol{\theta}_{it}\right)$.





\subsection{Bounded Reduction \texorpdfstring{(\Cref{col:bounded-reduction-s})}{}}

For ease of presentation, we first introduce the following lemma to show that $\lambda$, which is mentioned in \Cref{col:bounded-reduction-s}, is non-negative.

\begin{lemma}
Let assumptions (H1), (H2), (S1), and (S2) hold and suppose that the following two conditions hold:
\begin{enumerate}
    \item $\|\bm{\theta}^*\|_2 \leq 1$;
    \item $\vecalpha=(k-\gamma)\homomatrix\homomatrix^\top\bm{\theta}^*$ with $k,\gamma>0$.
\end{enumerate}
Then, we have
$$\lambda := \gamma\Big((\thetaao)^\top\homomatrix\homomatrix^\top\thetaao - (\thetastar)^\top\homomatrix\homomatrix^\top\thetastar\Big) \geq 0.$$
\end{lemma}

\begin{proof}
The second condition implies that $\thetaao=\arg\max_{\bm{\theta}_t:\|\bm{\theta}_t\|_2\leq1}\text{cPI}(\bm{\theta}_t)$, which follows from \Cref{col:max-improvement-s}.
Moreover, with the first condition, we now have $\text{cPI}(\thetastar) \leq \text{cPI}(\thetaao)$. Subsequently, that leads to
\begin{alignat*}{2}
&\ \text{cPI}(\thetastar) &&\leq \text{cPI}(\thetaao)
\\
&\ \gamma(\thetastar)^\top\homomatrix\homomatrix^\top\thetastar &&\leq \gamma(\thetaao)^\top\homomatrix\homomatrix^\top\thetastar
\\
\Rightarrow &\ (\homomatrix^\top\thetastar)^\top(\homomatrix^\top\thetastar) &&\leq (\homomatrix^\top\thetaao)^\top(\homomatrix^\top\thetastar)
\end{alignat*}
where we cancel out $\gamma$ from both sides as it is non-negative.

Next, by using the inequality $2\mathbf{u}^\top \mathbf{v}\leq\|\mathbf{u}\|^2+\|\mathbf{v}\|^2$ and treating $\homomatrix^\top\thetastar$ and $\homomatrix^\top\thetaao$ as $\mathbf{u}$ and $\mathbf{v}$, we obtain the desired result:
\begin{alignat*}{2}
&\ (\homomatrix^\top\thetastar)^\top(\homomatrix^\top\thetastar) &&\leq (\homomatrix^\top\thetaao)^\top(\homomatrix^\top\thetaao),
\end{alignat*}
which concludes the proof.
\end{proof}

We are now in a position to present the main proof of \Cref{col:bounded-reduction-s}.

\begin{proof}[Main Proof]
It follows from the Lipschitz continuity of $\widetilde{\bm{\delta}}(\hat{y})$ that, for an arbitrary agent with the baseline $\B_t$, we have
\begin{align*}
|\xi(\thetastar)-\xi(\thetaao)| &= |\widetilde{\bm{\delta}}(\hat{y}(\thetastar;\B_t)) - \widetilde{\bm{\delta}}(\hat{y}(\thetaao;\B_t))| 
\\
&\leq L|\hat{y}(\thetastar;\B_t)-\hat{y}(\thetaao;\B_t)|.
\end{align*}
This implies that
\begin{align*}
\xi(\thetastar)-\xi(\thetaao) &\leq |\xi(\thetastar)-\xi(\thetaao)|
\\
&\leq L|\hat{y}(\thetastar;\B_t)-\hat{y}(\thetaao;\B_t)|.
\end{align*}
Since $\widetilde{\bm{\delta}}(\hat{y}_t)$ is an increasing function of $\hat{y}_t$, then if $\hat{y}(\thetastar;\B_t)<\hat{y}(\thetaao;\B_t)$, we obtain the following for any arbitrarily large number $M\in\mathbb{R}^+$:
\begin{align*}
&\widetilde{\bm{\delta}}(\hat{y}(\thetastar;\B_t)) - \widetilde{\bm{\delta}}(\hat{y}(\thetaao;\B_t)) < 0
\\
\Rightarrow\ &\xi(\thetastar)-\xi(\thetaao) < 0
\\
\Rightarrow\ &p(\xi(\thetastar)-\xi(\thetaao)>M) = 0
.
\end{align*}

Next, we consider when there is a reduction in an agent's prediction, i.e., $\hat{y}(\thetastar;\B_t)>\hat{y}(\thetaao;\B_t)$. In this case, we have that
\begin{align*}
\xi(\thetastar)-\xi(\thetaao) &\leq L|\hat{y}(\thetastar;\B_t)-\hat{y}(\thetaao;\B_t)|
\\
&= L(\hat{y}(\thetastar;\B_t)-\hat{y}(\thetaao;\B_t)).
\end{align*}
When $\B_t\sim\mathcal{N}(0,\sigma^2I)$, we obtain
\begin{align*}
&\operatorname{Var}\big(\B_t^\top(\thetaao-\thetastar)\big) = \sigma^2\|\thetaao-\thetastar\|_2^2
\\
\Rightarrow\ &\B_t^\top(\thetaao-\thetastar)\sim\mathcal{N}(0,\ \sigma^2\|\thetaao-\thetastar\|_2^2)
\\
\Rightarrow\ &\frac{\B_t^\top(\thetaao-\thetastar)}{\sigma\|\thetaao-\thetastar\|_2}\sim\mathcal{N}(0,1).
\end{align*}
As a result, the counterfactual change in predicted outcome of an arbitrary agent can be expressed as follows:
\begin{align*}
\hat{y}(\thetaao;\B_t)-\hat{y}(\thetastar;\B_t)
=\ &(\B_t+\homomatrix\mathbf{a}_{\bm{\theta}_t=\thetaao})^\top\thetaao - (\B_t+\homomatrix\mathbf{a}_{\bm{\theta}_t=\thetastar})^\top\thetastar
\\
=\ &\B_t^\top(\thetaao-\thetastar) + \big(\mathbf{a}_{\bm{\theta}_t=\thetaao}^\top\homomatrix^\top\thetaao - \mathbf{a}_{\bm{\theta}_t=\thetastar}^\top\homomatrix^\top\thetastar\big)
\\
=\ &\B_t^\top(\thetaao-\thetastar) + \underbrace{\big(\gamma(\thetaao)^\top\homomatrix\homomatrix^\top\thetaao - \gamma(\thetastar)^\top\homomatrix^\top\thetastar\big)}_{\geq\ 0}
\\
=\ &\B_t^\top(\thetaao-\thetastar) + \lambda .
\end{align*}
Then, we obtain the counterfactual reduction in admission chance as follows: 
\begin{align*}
\xi(\thetastar)-\xi(\thetaao) &\leq L\left|\hat{y}\big(\thetastar\mid\mathbf{b}_t\big)-\hat{y}\big(\thetaao\mid\mathbf{b}_t\big)\right|
\\
&= L\left(\hat{y}\big(\thetastar\mid\mathbf{b}_t\big)-\hat{y}\big(\thetaao\mid\mathbf{b}_t\big)\right)
\\
&= L\left(\B_t^\top(\thetastar-\thetaao) - \lambda\right) .
\end{align*}
Subsequently, for any arbitrarily large $M\in\mathbb{R}^+$, it follows that
\begin{align*}
p\big(\xi(\thetastar)-\xi(\thetaao)>M\big) 
&\leq p\big(L\left(\B_t^\top(\thetastar-\thetaao)-\lambda\right)>M\big)
\\
&= p\big(\B_t^\top(\thetastar-\thetaao)>M/L+\lambda\big)
\\
&= p\left(\frac{\B_t^\top(\thetastar-\thetaao)}{\sigma\|\thetaao-\thetastar\|_2}>\frac{M/L+\lambda}{\sigma\|\thetaao-\thetastar\|_2}\right)
\\
&= 1-\operatorname{CDF}\left(\frac{M/L+\lambda}{\sigma\|\thetaao-\thetastar\|_2}\right)
\\
&= 1-\Phi\left(\frac{M/L+\lambda}{\sigma\|\thetaao-\thetastar\|_2}\right)
\\
&= \Phi\left(\frac{-M/L-\lambda}{\sigma\|\thetaao-\thetastar\|_2}\right).
\end{align*}
This concludes the proof.
\end{proof}

\subsection{Local Exogeneity \texorpdfstring{(\Cref{theorem:local-exo-s})}{}}

For the sake of clarity, we introduce the following lemma that will be used later in the main proof.

\begin{lemma}\label{lemma:no-shift-s}
Suppose that there exists a pair of rounds $t$ and $t^\prime$ such that $\bm{\theta}_t=k\bm{\theta}_{t^\prime}$ for some $k>0$. Then, there is no shift in the conditional distribution of selected agents within those two rounds. That is,
\begin{align*}
p(\B_{t}=\mathbf{b}_\diamond\mid W_{t}=1) &= p(\B_{t^\prime}=\mathbf{b}_\diamond\mid W_{t^\prime}=1), \forall \mathbf{b}_\diamond,
\\
p(O_{t}=o_\diamond\mid W_{t}=1) &= p(O_{t^\prime}=o_\diamond\mid W_{t^\prime}=1), \forall o_\diamond .
\end{align*}
\end{lemma}
\begin{proof}
We refer the readers to the corresponding proof in the multiple decision makers setting (\Cref{apx:local-exo-m}), as it can be trivially translated backwards. 
\end{proof}

We are now in a position to present the main proof of \Cref{theorem:local-exo-s}.

\begin{proof}[Main Proof]
Using the result in \Cref{lemma:no-shift-s}, when there exist two rounds $t$ and $t^\prime$ such that $\bm{\theta}_t=k\bm{\theta}_{t^\prime}$ for some $k>0$, we obtain 
\begin{align*}
\mathbb{E}\left[Y_t \,\vert\, W_t=1;\  \bm{\theta}_t\right] - \mathbb{E}\left[Y_{t^\prime} \,\vert\, W_{t^\prime}=1;\ \bm{\theta}_{t^\prime}\right]
=\ &\Big(\mathbb{E}\left[\X_{t} \,\vert\, W_t=1;\  \bm{\theta}_t\right]-\mathbb{E}\left[\X_{t^\prime} \,\vert\, W_{t^\prime}=1;\  \bm{\theta}_{t^\prime}\right]\Big)^\top\bm{\theta}^*
\\
&\hspace{5mm} + \underbrace{\Big(\mathbb{E}\left[O_t\mid W_t=1;\ \bm{\theta}_t\right] - \mathbb{E}\left[O_{t^\prime}\mid W_{t^\prime}=1;\ \bm{\theta}_{t^\prime}\right]\Big)}_{=0}.
\end{align*}%
This concludes the proof.
\end{proof}

\subsection{Dominant Strategy (Proposition \ref{prop:dom-strat})}
\label{apx:dom-strat}
We rewrite $\mathcal{Q}_i(\bm{\theta}_t^\text{all})$ respectively as an addition of two functions of $\bm{\theta}_{it}$ and $\bm{\theta}_t^{-i}$:
\begin{align*}
\mathcal{Q}_i(\{\bm{\theta}_{it},\bm{\theta}_t^{-i}\}) &= \text{cBP}_i(\{\bm{\theta}_{it},\bm{\theta}_t^{-i}\}) + \text{cPI}_i(\{\bm{\theta}_{it},\bm{\theta}_t^{-i}\}) \\
&= g_i(\bm{\theta}_{it}) + h_i(\bm{\theta}_t^{-i}) + c_i + \Big(\sum^n_{j=1}\gamma_j\bm{\theta}_{jt}\Big)^\top\homomatrix\homomatrix^\top\bm{\theta}_i^* \\
&= \Bigg(g_i(\bm{\theta}_{it}) + \gamma_i\bm{\theta}_{it}^\top\homomatrix\homomatrix^\top\bm{\theta}_i^*\Bigg)  
+ \Bigg(h_i(\bm{\theta}_t^{-i}) + \Big(\sum_{j\neq i}\gamma_j\bm{\theta}_{jt}\Big)^\top\homomatrix\homomatrix^\top\bm{\theta}_i^* + c_i\Bigg) \\
&= \phi_i(\bm{\theta}_{it}) + \psi_i(\bm{\theta}_t^{-i})
\end{align*}
where we group the terms on the right-hand side together and further denote them as $\phi_i$ and $\psi_i$. By Assumption (M1), the functional forms of $g_i$, $h_i$, and the constant $c_i$ are fixed regardless of $\bm{\theta}_t^{-i}$. Thus, we obtain the same result where the functional forms of $\phi_i$ and $\psi_i$ are independent of $\bm{\theta}_t^{-i}$. 

Let $\bm{\theta}_{it}=\bm{\theta}_i^\dagger$ be a maximiser for $\mathcal{Q}_i(\{\bm{\theta}_{it},\bm{\theta}_{\diamond}^{-i}\})$ where the parameter $\bm{\theta}_t^{-i}$ takes on the value $\bm{\theta}_\diamond^{-i}$. Then,
\begin{eqnarray*}
&\mathcal{Q}_i(\{\bm{\theta}_i^\dagger,\bm{\theta}_{\diamond}^{-i}\}) - \mathcal{Q}_i(\{\bm{\theta}_{it},\bm{\theta}_{\diamond}^{-i}\}) &\geq 0, \forall \bm{\theta}_{it} \\
\Rightarrow &\phi_i(\bm{\theta}_i^\dagger)+\psi_i(\bm{\theta}_\diamond^{-i}) - \phi_i(\bm{\theta}_{it})-\psi_i(\bm{\theta}_\diamond^{-i}) &\geq 0, \forall \bm{\theta}_{it} \\
\Rightarrow &\phi_i(\bm{\theta}_i^\dagger) - \phi_i(\bm{\theta}_{it}) &\geq 0, \forall \bm{\theta}_{it}
.
\end{eqnarray*}

Then, for all $\bm{\theta}_{it}$ and $\bm{\theta}_t^{-i}$, we have
\begin{align*}
&\Big(\phi_i(\bm{\theta}_i^\dagger) + \psi_i(\bm{\theta}_t^{-i})\Big) - \Big(\phi_i(\bm{\theta}_{it}) + \psi_i(\bm{\theta}_t^{-i})\Big) \geq 0,
\end{align*}
which implies that
\begin{align*}
\mathcal{Q}_i(\{\bm{\theta}_i^\dagger,\bm{\theta}_t^{-i}\}) - \mathcal{Q}_i(\{\bm{\theta}_{it},\bm{\theta}_t^{-i}\}) \geq 0 .
\end{align*}
Therefore, $\bm{\theta}_{it}=\bm{\theta}_i^\dagger$ is also the maximiser for all functions in the family of $\mathcal{Q}_i(\{\bm{\theta}_{it}, \bm{\theta}_t^{-i}\})$, that is parameterised by $\bm{\theta}_t^{-i}$.

\subsection{Bounded Optimum, extended \texorpdfstring{(\Cref{theorem:bounded-optim-m})}{}}
\label{apx:bounded-opt-ext}

For ease of presentation, we introduce the following lemma that will later be used by the main proof.

\begin{lemma}\label{lemma:linear-optimum}
The maximiser of the linear function $\vecalpha^\top\bm{\theta}+\beta$, subject to the constraint $\|\boldsymbol{\theta}\|_2\leq 1$, has the form $\vecalpha/\|\vecalpha\|_2$.
\end{lemma}
\begin{proof}
The result follows either by applying the KKT conditions with the equivalent constraint $\|\bm{\theta}\|^2_2\leq1$, or by using geometry where $\max_{\bm{\theta}}\vecalpha^\top\bm{\theta} = \max_{(\|\bm{\theta}\|,\cos(\vecalpha,\bm{\theta}))}\|\vecalpha\|\|\bm{\theta}\|\cos(\vecalpha,\bm{\theta})$.
\end{proof}

We are now in a position to present the main proof of \Cref{theorem:bounded-optim-m}.

\begin{proof}[Main Proof]
Consider an arbitrary decision maker $i$ for whom Assumptions (M1), (M2), and (M3) hold. 
By \Cref{prop:dom-strat}, all objective functions in the family $\mathcal{Q}_i(\{\bm{\theta}_{it},\bm{\theta}_t^{-i}\})$ have a common maximiser regardless of $\bm{\theta}_t^{-i}$. Hence, it is sufficient to analyse an objective function for an arbitrary value of $\bm{\theta}_t^{-i}$. To this end, we have
\begin{align*}
\mathcal{Q}_i(\{\bm{\theta}_{it},\bm{\theta}_t^{-i}\}) &= \text{cBP}_i(\{\bm{\theta}_{it},\bm{\theta}_t^{-i}\}) + \text{cPI}_i(\{\bm{\theta}_{it},\bm{\theta}_t^{-i}\}) \\
&= g_i(\bm{\theta}_{it}) + h_i(\bm{\theta}_t^{-i}) + c_i + \Big(\sum^n_{j=1}\gamma_j\bm{\theta}_{jt}\Big)^\top\homomatrix\homomatrix^\top\bm{\theta}_i^* \\
&= \Bigg(g_i(\bm{\theta}_{it}) + \gamma_i\bm{\theta}_{it}^\top\homomatrix\homomatrix^\top\bm{\theta}_i^*\Bigg) 
+ \Bigg(\underbrace{h_i(\bm{\theta}_t^{-i}) + \Big(\sum_{j\neq i}\gamma_j\bm{\theta}_{jt}\Big)^\top\homomatrix\homomatrix^\top\bm{\theta}_i^* + c_i}_{\psi_i(\bm{\theta}_t^{-i})}\Bigg) \\
&= \Big(\vecalpha_i^\top\bm{\theta}_{it}+\beta_i + \gamma_i\bm{\theta}_{it}^\top\homomatrix\homomatrix^\top\bm{\theta}_i^*\Big) + \psi_i(\bm{\theta}_t^{-i}) \\
&= \big(\vecalpha_i+\gamma_i\homomatrix\homomatrix^\top\bm{\theta}_i^*\big)^\top\bm{\theta}_{it} + \Big(\beta_i + \psi_i(\bm{\theta}_t^{-i})\Big).
\end{align*}
Then, by \Cref{lemma:linear-optimum}, 
we obtain the maximiser $\bm{\theta}_{it}=\bm{\theta}_i^\text{AO}$ where
$$\bm{\theta}_i^\text{AO} = \frac{\vecalpha_i+\gamma_i\homomatrix\homomatrix^\top\bm{\theta}_i^*}{\|\vecalpha_i+\gamma_i\homomatrix\homomatrix^\top\bm{\theta}_i^*\|_2}.$$
This concludes the proof.
\end{proof}

\subsection{Maximum Improvement, extended}
\label{apx:max-imp-ext}

\begin{proof}
Consider an arbitrary decision maker $i$ for whom Assumptions (M1), (M2), and (M3) hold. 
Recall that we can decompose its parameterised objective into cBP$_i$ and cPI$_i$ as follows:
\begin{align*}
&\mathcal{Q}_i(\{\bm{\theta}_{it},\bm{\theta}_t^{-i}\}) = \text{cBP}_i(\{\bm{\theta}_{it},\bm{\theta}_t^{-i}\}) + \text{cPI}_i(\{\bm{\theta}_{it},\bm{\theta}_t^{-i}\}),
\end{align*}
where each of those two terms is a linear function of $\bm{\theta}_{it}$, and regardless of the value $\bm{\theta}_t^{-t}$, they take the following forms:
\begin{align*}
\text{cBP}_i(\{\bm{\theta}_{it},\bm{\theta}_t^{-i}\})\ &=\ g_i(\bm{\theta}_{it}) + h_i(\bm{\theta}_t^{-i}) + c_i\ =\ \vecalpha_i^\top\bm{\theta}_{it} + \beta_i + h_i(\bm{\theta}_t^{-i}) + c_i.
\\
\text{cPI}_i(\{\bm{\theta}_{it},\bm{\theta}_t^{-i}\})\ &=\ \Big(\sum^n_{j=1}\gamma_j\bm{\theta}_{jt}\Big)^\top\homomatrix\homomatrix^\top\bm{\theta}_i^*\ =\ \gamma_i\bm{\theta}_{it}^\top\homomatrix\homomatrix^\top\bm{\theta}_i^* + \Big(\sum_{j\neq i}\gamma_j\bm{\theta}_{jt}\Big)^\top\homomatrix\homomatrix^\top\bm{\theta}_i^*.
\end{align*}%

Furthermore, because the functional forms of $g_i$ , $h_i$, and the constant $c_i$ are not affected by $\bm{\theta}_{it}$ and $\bm{\theta}_t^{-i}$, we can find the maximisers (when $\gamma_i > 0$) for these linear functions using the similar technique in \Cref{prop:dom-strat} and \Cref{theorem:bounded-optim-m}. That is,
\begin{align*}
\frac{\vecalpha_i}{\|\vecalpha_i\|_2} &= {\arg\max}_{\bm{\theta}_{it}:\|\bm{\theta}_{it}\|_2\leq1} \text{cBP}_i(\{\bm{\theta}_{it},\bm{\theta}_t^{-i}\}), \forall \bm{\theta}_t^{-i},
\\
\frac{\homomatrix\homomatrix^\top\thetastar_i}{\|\homomatrix\homomatrix^\top\thetastar_i\|_2} &= {\arg\max}_{\bm{\theta}_{it}:\|\bm{\theta}_{it}\|_2\leq1} \text{cPI}_i(\{\bm{\theta}_{it},\bm{\theta}_t^{-i}\}), \forall \bm{\theta}_t^{-i} .
\end{align*}
When $\vecalpha_i = (k_i-\gamma_i)\homomatrix\homomatrix^\top\thetastar_i$ for some $k_i>0$, $\thetaao_i$ also maximises the improvements, cPI$_i(\{\bm{\theta}_{it},\bm{\theta}_t^{-i}\})$ for all $\bm{\theta}_t^{-i}$, i.e.,
\begin{align*}
\thetaao_i = \frac{\vecalpha_i+\gamma_i\homomatrix\homomatrix^\top\thetastar_i}{\|\vecalpha_i+\gamma_i\homomatrix\homomatrix^\top\thetastar_i\|_2} = \frac{k_i\homomatrix\homomatrix^\top\thetastar_i}{\|k_i\homomatrix\homomatrix^\top\thetastar_i\|_2}
= sgn(k_i)\frac{\homomatrix\homomatrix^\top\thetastar_i}{\|\homomatrix\homomatrix^\top\thetastar_i\|_2}
= \frac{\homomatrix\homomatrix^\top\thetastar_i}{\|\homomatrix\homomatrix^\top\thetastar_i\|_2}
\end{align*}
where $sgn(\cdot)$ refers to the sign function.
\end{proof}

\subsection{Local Exogeneity, extended \texorpdfstring{(\Cref{theorem:local-exo-m})}{}}
\label{apx:local-exo-m}

For clarity, we introduce the following lemma that will be used later in the main proof.

\begin{lemma}\label{lemma:no-shift-m}
Suppose that all decision makers follow the cooperative protocol (\Cref{def:protocol}) in a pair of two rounds $t_1$ and $t_2$. Then, there is no shift in the conditional distribution of enrolled agents within those two rounds. 
Specifically, for any arbitrary decision maker $i$, we have that
\begin{align*}
p(\B_{t}=\mathbf{b}_\diamond\mid Z_{t}=i) &= p(\B_{t^\prime}=\mathbf{b}_\diamond\mid Z_{t^\prime}=i), \forall \mathbf{b}_\diamond\in\mathbb{R}^m,
\\
p(O_{it}=o_\diamond\mid Z_{t}=i) &= p(O_{it^\prime}=o_\diamond\mid Z_{t^\prime}=i), \forall o_\diamond\in\mathbb{R}.
\end{align*}
\end{lemma}
\begin{proof}
At an arbitrary round $t$, any student arriving at this round will have the same best response $\mathbf{a}_t$ because of the assumption of common effort conversion $\mathcal{E}_t=\homomatrix$. Let $(\mathbf{b}_\diamond, \mathbf{x}_\diamond)$ be any realisation of the pair of random variables $(\B_t, \X_t\,|\,\B_t=\mathbf{b}_\diamond)$, we rewrite each selection function $\bm{\delta}_{\bm{\theta}_{\mathit{jt}}}$ from being a function of $\mathbf{x}_\diamond$ into a function of $\mathbf{b}_\diamond$. That is, for all $j \in [n]$,
\begin{eqnarray*}
 \bm{\delta}_{\bm{\theta}_{\mathit{jt}}}(\mathbf{x}_\diamond) &=& p(\X_t^\top\bm{\theta}_{\mathit{jt}}\leq\mathbf{x}_\diamond^\top\bm{\theta}_{\mathit{jt}})
\\
&=& p\big((\B_t+\homomatrix\mathbf{a}_t)^\top\bm{\theta}_{\mathit{jt}} \leq (\mathbf{b}_\diamond+\homomatrix\mathbf{a}_t)^\top\bm{\theta}_{\mathit{jt}}\big)
\\
&=& p\big((\B_t-\mathbf{b}_\diamond)^\top\bm{\theta}_{\mathit{jt}} \leq 0\big)
\\
&=& p(W_{\mathit{jt}}=1\mid\B_t=\mathbf{b}_\diamond;\ \bm{\theta}_{\mathit{jt}}).
\end{eqnarray*}

When $\bm{\theta}_{jt^\prime}=k_j\bm{\theta}_{jt}$ with $k_j>0$ for all $j\in[n]$, the conditional distribution of admission remains the same in those two rounds, as shown below. For all $j\in[n]$,
\begin{align*}
p(W_{jt^\prime}=1\mid\B_{t^\prime}=\mathbf{b}_\diamond;\ \bm{\theta}_{jt^\prime})
&=\ p\big((\B_{t^\prime}-\mathbf{b}_\diamond)^\top\bm{\theta}_{jt^\prime} \leq 0\big)
\\
&=\ p\big((\B_{t^\prime}-\mathbf{b}_\diamond)^\top(k_j\bm{\theta}_{jt}) \leq 0\big)
\\
&=\ p\big((\B_{t}-\mathbf{b}_\diamond)^\top\bm{\theta}_{jt} \leq 0\big)
\\
&=\ p(W_{jt}=1\mid\B_{t}=\mathbf{b}_\diamond;\ \bm{\theta}_{jt})
\end{align*}
where we can replace $\B_{t^\prime}$ with $\B_t$ because they have the same marginal distribution.

Since all agents have the same compliance behaviour where $Z_t\indep\{\X_t,\B_t\}\mid\{W_{jt}\}_{j=1}^n$, we obtain the following for any arbitrary decision maker $i$:
$p(Z_{t^\prime}=i\mid\{W_{jt^\prime}=w_j^\diamond\}_{j=1}^n)
 = p(Z_{t}=i\mid\{W_{jt}=w_j^\diamond\}_{j=1}^n)$
where $w_j^\diamond$ denotes a realisation of the binary variable $W_{jt}$.

Putting together the results we have shown in this proof about the admission and compliance probabilities, we obtain the following:
\begin{align*}
p(Z_{t^\prime}=i\mid\B_{t^\prime}=\mathbf{b}_\diamond) 
&= \sum_{\{w_j^\diamond\}_{j=1}^n\in\mathcal{W}^\text{all}}\Bigg[p(Z_{t^\prime}=i\mid\{W_{jt^\prime}=w_j^\diamond\}_{j=1}^n) \prod_{j=1}^n p(W_{jt^\prime}=w_j^\diamond\mid\B_{t^\prime}=\mathbf{b}_\diamond;\ \bm{\theta}_{jt^\prime})\Bigg]
\\
&= \sum_{\{w_j^\diamond\}_{j=1}^n\in\mathcal{W}^\text{all}}\Bigg[p(Z_{t}=i\mid\{W_{jt}=w_j^\diamond\}_{j=1}^n) \prod_{j=1}^n p(W_{jt}=w_j^\diamond\mid\B_{t}=\mathbf{b}_\diamond;\ \bm{\theta}_{jt})\Bigg]
\\
&= p(Z_{t}=i\mid\B_{t}=\mathbf{b}_\diamond),
\end{align*}
where $\mathcal{W}^\text{all}$ denotes the space containing all possible $n$-dimensional binary vectors.
We then use the Bayes' theorem to derive the desired result as follows:
\begin{align*}
p(\B_{t^\prime}=\mathbf{b}_\diamond\mid Z_{t^\prime}=i) &\propto p(Z_{t^\prime}=i\mid\B_{t^\prime}=\mathbf{b}_\diamond) p(\B_{t^\prime}=\mathbf{b}_\diamond)
\\
\Rightarrow p(\B_{t^\prime}=\mathbf{b}_\diamond\mid Z_{t^\prime}=i) &= p(\B_{t}=\mathbf{b}_\diamond\mid Z_{t}=i).
\end{align*}
This result has been shown for an arbitrary decision maker $i$ and thus, it also holds for all decision makers. Then, we can also obtain the result for $p(O_{it}=o_\diamond\mid Z_t=i)$ in a similar way: For all $o_\diamond$, 
\begin{align*}
p(O_{it^\prime}=o_\diamond\mid Z_{t^\prime}=i)
&\propto \int_{\mathbf{b}_{t^\prime}} p(Z_{t^\prime}=i\mid\mathbf{b}_{t^\prime})\ p(\mathbf{b}_{t^\prime}\mid O_{it^\prime}=o_\diamond)\ d\mathbf{b}_{t^\prime} p(O_{it^\prime}=o_\diamond)
\\
&\propto \int_{\mathbf{b}_{t}} p(Z_{t}=i\mid\mathbf{b}_{t})\ p(\mathbf{b}_{t}\mid O_{it}=o_\diamond)\ d\mathbf{b}_{t} p(O_{it}=o_\diamond)
\\
&\propto p(O_{it}=o_\diamond\mid Z_{t}=i)
.
\end{align*}
Then we obtain $p(O_{it^\prime}=o_\diamond\mid Z_{t^\prime}=i) = p(O_{it}=o_\diamond\mid Z_{t}=i)$, which concludes the proof.
\end{proof}

We are now in a position to present the main proof of \Cref{theorem:local-exo-m}.

\begin{proof}[Main Proof]
When the pair of rounds $t$ and $t^\prime$ satisfies the cooperative protocol (\Cref{def:protocol}), it follows from  \Cref{lemma:no-shift-m} that
\begin{align*}
\mathbb{E}\left[Y_{it^\prime} \,\vert\, Z_{t^\prime}=i\ ;\  \bm{\theta}^{\text{all}}_{t^\prime}\right] - \mathbb{E}\left[Y_{it} \,\vert\, Z_{t}=i\ ;\ \bm{\theta}^{\text{all}}_{t}\right] 
=\ &\Big(\mathbb{E}\left[\X_{t^\prime} \,\vert\, Z_{t^\prime}=i\ ;\  \bm{\theta}^{\text{all}}_{t^\prime}\right]-\mathbb{E}\left[\X_{t} \,\vert\, Z_{t}=i\ ;\  \bm{\theta}^{\text{all}}_{t}\right]\Big)^\top\bm{\theta}^*_i
\\
&\quad + \underbrace{\Big(\mathbb{E}\left[O_{it^\prime}\mid Z_{t^\prime}=i;\ \bm{\theta}_{t^\prime}^\text{all}\right] - \mathbb{E}\left[O_{it}\mid Z_{t}=i;\ \bm{\theta}_t^\text{all}\right]\Big),}_{=0}
\end{align*}%
which concludes the proof.
\end{proof}

\section{Instrumental Variable for Causal Estimation under Competitive Selection}
\label{apx:invalid-iv}
With the selection procedure, one can only observe the outcomes of selected agents, thus obtaining the biased data set $D=\{\bm{\theta}_t,\mathbf{x}_t,\{y_t\mid w_t=1\}\}^T_{t=1}$. 
Since conditioning on the collider $W_t$ renders $\bm{\theta}_t$ and $\B_t$ dependent, this leads to $\bm{\theta}_t\dep Y_t\mid W_t$ via the noise $O_t$. 
As a result, we cannot use $\bm{\theta}_t$ as an instrumental variable in our setting, unlike \citet{harris2022strategic}), because the unconfoundedness assumption required for $\bm{\theta}_t$ to be a valid IV is violated.

One could also think of grouping the two variables $(\X_t, W_t)$ together as a treatment variable to avoid explicitly conditioning on the collider $W_t$. However, this does not work out trivially since any additive noise model of the form $Y_t:=f(\X_t,W_t)+N$ implies that we have biased observation of the noise where $\mathbb{E}_D\left[N\mid\bm{\theta}_t\right]\neq\mathbb{E}_D\left[N\right]$. 
This is because both the noise $N$ and the supposed instrumental variable $\bm{\theta}_t$ must be correlated with the treatment $(\X_t, W_t)$ by design but the data set $D$ only contains the case where $W_t=1$. 
Thus, $\bm{\theta}_t$ still cannot serve as a valid instrumental variable.

\section{Connection to the Maximin Objective}\label{apx:maximin}
In this subsection, we show that when all objective functions, within the family $\mathcal{Q}_i(\{\bm{\theta}_{it},\bm{\theta}_t^{-i}\})$ of an arbitrary decision maker $i$, have a common maximiser $\bm{\theta}_{it}=\bm{\theta}_i^\dagger$, it is also a solution to the maximin strategy of the decision maker $i$.

\begin{definition}[Maximin Expected Utility] \label{def:maximin-exp}
Let $i$ be an arbitrary decision maker and $\mathcal{Q}_i(\{\bm{\theta}_{it},\bm{\theta}_t^{-i}\})$ be the family of their utility functions that are parameterised by $\bm{\theta}_t^{-i}$ which are the parameters controlled by their rival decision makers. 
We denote $P_\text{rivals}$ as the conditional distribution $P_{\bm{\theta}_t^{-i}\mid Z_t=i,\ \bm{\theta}_{it}}$ and  $\widetilde{\mathcal{Q}}_i(\bm{\theta}_{it},P_\text{rivals})$ as the expected utility of this decision maker where
$$\widetilde{\mathcal{Q}}_i(\bm{\theta}_{it},P_\text{rivals}) = \mathbb{E}_{P_\text{rivals}}\left[\mathcal{Q}_i(\{\bm{\theta}_{it},\bm{\theta}_t^{-i}\})\mid Z_t=i,\ \bm{\theta}_{it}\right].$$
Then, the maximin expected utility objective is defined as
$$\max_{\bm{\theta}_{\mathit{it}}}\min_{P_\text{rivals}}\;\widetilde{\mathcal{Q}}_i(\bm{\theta}_{it},P_\text{rivals}).$$
\end{definition}

\begin{lemma}[]
Let $i$ be an arbitrary decision maker for whom Assumption (M1) holds and suppose that all the objective functions within the family $\mathcal{Q}_i(\{\bm{\theta}_{it},\bm{\theta}_t^{-i}\})$ have a common maximiser $\bm{\theta}_{it}=\bm{\theta}_i^\dagger$. 
When there is no constraint on the behaviour of $P_\text{rivals}$ and a solution to the maximin objective (Def. \ref{def:maximin-exp}) exists, then this solution contains $\bm{\theta}_{it}=\bm{\theta}_i^\dagger$, i.e.,
$$\bm{\theta}_i^\dagger\in\arg\max_{\bm{\theta}_{it}}\min_{P_\text{rivals}}\;\widetilde{\mathcal{Q}}_i(\bm{\theta}_{it},P_\text{rivals}).$$
\end{lemma}

The condition of $P_\text{rivals}$ in this lemma implies the worst-case scenario where all rivals of the decision maker $i$ cooperate to minimise its objective. We show the proof for the discrete case of $\bm{\theta}_t^{-i}$ below.
\begin{proof}
Under Assumption (M1), recall from \Cref{apx:dom-strat} 
that each objective function $\mathcal{Q}_i(\{\bm{\theta}_{it},\bm{\theta}_t^{-i}\})$ can be decomposed into $\phi_i(\bm{\theta}_{it})$ and $\psi_i(\bm{\theta}_t^{-i})$, then
\begin{align*}
\widetilde{\mathcal{Q}}_i(\bm{\theta}_{it},P_\text{rivals}) =\ &\mathbb{E}_{P_\text{rivals}}\left[\mathcal{Q}_i(\{\bm{\theta}_{it},\bm{\theta}_t^{-i}\})\mid Z_t=i,\ \bm{\theta}_{it}\right] \\
=\ &\mathbb{E}_{P_\text{rivals}}\left[\phi_i(\bm{\theta}_{it})+\psi_i(\bm{\theta}_t^{-i})\mid Z_t=i,\ \bm{\theta}_{it}\right] \\
=\ &\phi_i(\bm{\theta}_{it})+\mathbb{E}_{P_\text{rivals}}\left[\psi_i(\bm{\theta}_t^{-i})\mid Z_t=i,\ \bm{\theta}_{it}\right]
.
\end{align*}
For any value of $\bm{\theta}_{it}$, we have
\begin{align*}
\min_{P_\text{rivals}}\widetilde{\mathcal{Q}}_i(\bm{\theta}_{it},P_\text{rivals})
&=\min_{P_\text{rivals}}\Big\{\phi_i(\bm{\theta}_{it})+\mathbb{E}_{P_\text{rivals}}\left[\psi_i(\bm{\theta}_t^{-i})\mid Z_t=i,\ \bm{\theta}_{it}\right]\Big\} \\
&=
\phi_i(\bm{\theta}_{it})+\min_{P_\text{rivals}}\mathbb{E}_{P_\text{rivals}}\left[\psi_i(\bm{\theta}_t^{-i})\mid Z_t=i,\ \bm{\theta}_{it}\right].
\end{align*}
In the discrete case and when there is no constraint on $P_\text{rivals}$, it can be seen easily that
\begin{align*}
\min_{P_\text{rivals}}\mathbb{E}_{P_\text{rivals}}\left[\psi_i(\bm{\theta}_t^{-i})\mid Z_t=i,\ \bm{\theta}_{it}\right]
&= 
\min_{P_\text{rivals}}\sum_{\bm{\theta}_t^{-i}}\psi_i(\bm{\theta}_t^{-i})\ p(\bm{\theta}_t^{-i}\mid Z_t=i,\ \bm{\theta}_{it}) \\
&=\min_{\bm{\theta}_t^{-i}}\psi_i(\bm{\theta}_t^{-i})
.
\end{align*}
Putting all results so far in this proof together yields
\begin{align*}
\min_{P_\text{rivals}}\widetilde{\mathcal{Q}}_i(\bm{\theta}_{it},P_\text{rivals})
&= \phi_i(\bm{\theta}_{it})+\min_{P_\text{rivals}}\mathbb{E}_{P_\text{rivals}}\left[\psi_i(\bm{\theta}_t^{-i})\mid Z_t=i,\ \bm{\theta}_{it}\right] \\
&= \phi_i(\bm{\theta}_{it})+\min_{\bm{\theta}_t^{-i}}\psi_i(\bm{\theta}_t^{-i}) \\
&= \min_{\bm{\theta}_t^{-i}}\mathcal{Q}_i(\{\bm{\theta}_{it}, \bm{\theta}_t^{-i}\}).
\end{align*}

When a maximin solution exists, we obtain the equivalence in objectives because the inner minimisation parts on both sides are equivalent for all values of $\bm{\theta}_{it}$, i.e.,
\begin{align*}
\arg\max_{\bm{\theta}_{it}}\min_{P_\text{rivals}}\widetilde{\mathcal{Q}}_i(\bm{\theta}_{it},P_\text{rivals})
= \arg\max_{\bm{\theta}_{it}}\min_{\bm{\theta}_t^{-i}}\mathcal{Q}_i(\{\bm{\theta}_{it}, \bm{\theta}_t^{-i}\}).
\end{align*}%
It can be seen that a solution to the r.h.s contains $\bm{\theta}_{it}=\bm{\theta}_i^\dagger$ because for any pair of $\{\bm{\theta}_{it},\bm{\theta}_t^{-i}\}$, we can always obtain better value for $\mathcal{Q}_i$ by replacing $\bm{\theta}_{it}$ with $\bm{\theta}_i^\dagger$. Finally, because the inner minimisation parts are equivalent, a solution to the l.h.s must also contain $\bm{\theta}_{it}=\bm{\theta}_i^\dagger$.
\end{proof}

\section{Bounded Reduction, extended \texorpdfstring{(\Cref{col:bounded-reduction-m})}{}}
\label{apx:bounded-reduction-m}

Before showing the proof for the corollary, we first demonstrate cases where the 2nd and the 3rd conditions in \Cref{col:bounded-reduction-m}
can hold simultaneously. Consider cases where $\thetastar_j=r_j\thetastar_i$ for some $r_j\in\mathbb{R}^+$ for all $j\in[n]$ then the 2nd condition leads to the 3rd condition. 

\begin{proof}[Proof sketch]
When $\vecalpha_j=(k_j-\gamma_j)\homomatrix\homomatrix^\top\thetastar_j$ with $k_j,\gamma_j>0$, for all $j\in[n]$, we obtain  $\thetaao_j\propto\homomatrix\homomatrix^\top\thetastar_j$ (\Cref{col:max-improvement-m}). Using $\thetastar_j=r_j\thetastar_i$, we get $\thetaao_j\propto\thetaao_i$ which implies $\thetaao_j=\thetaao_i$.

Finally, using the techniques in the proof of \Cref{col:bounded-reduction-s}, we arrive at
$$(\bm{\theta}_{jt})^\top\homomatrix\homomatrix^\top(\thetaao_i-\thetastar_i)^\top \geq0 \quad \forall \bm{\theta}_{jt}\in\{\thetaao_j,\thetastar_j\},$$
which concludes the proof sketch.
\end{proof}

We then show an example where there does not exist any $r_j$ such that $\thetastar_j=r_j\thetastar_i$ but both conditions (2nd and 3rd) still hold.
\begin{example}
Let there be two decision makers with $\thetastar_1=[0.4,0.9]^\top$, $\thetastar_2=[0.6,0.7]^\top$, and let $\homomatrix=\begin{bmatrix} 2 & 0 \\ 0 & 1 \end{bmatrix}$.

When $\vecalpha_j=(k_j-\gamma_j)\homomatrix\homomatrix^\top\thetastar_j$ with $k_j,\gamma_j>0 \quad \forall j\in[n]$, $\thetaao_j=\frac{\homomatrix\homomatrix^\top\thetastar_j}{\|\homomatrix\homomatrix^\top\thetastar_j\|_2}$ (see \Cref{col:max-improvement-m}). 

Let $\Lambda:=\homomatrix\homomatrix^\top(\thetaao_1-\thetastar_1)=\homomatrix\homomatrix^\top\left(\frac{\homomatrix\homomatrix^\top\thetastar_1}{\|\homomatrix\homomatrix^\top\thetastar_1\|_2}-\thetastar_1\right)$. 

Then, plug in the values and see that the 3rd condition holds:
\begin{align*}
\left\{\begin{aligned}
(\thetastar_1)^\top\homomatrix\homomatrix^\top(\thetaao_1-\thetastar_1) &= (\thetastar_1)^\top\Lambda > 0,
\\
(\thetastar_2)^\top\homomatrix\homomatrix^\top(\thetaao_1-\thetastar_1) &= (\thetastar_2)^\top\Lambda > 0,
\\
(\thetaao_1)^\top\homomatrix\homomatrix^\top(\thetaao_1-\thetastar_1) &= \left(\frac{\homomatrix\homomatrix^\top\thetastar_1}{\|\homomatrix\homomatrix^\top\thetastar_1\|_2}\right)^\top\Lambda > 0,
\\
(\thetaao_2)^\top\homomatrix\homomatrix^\top(\thetaao_1-\thetastar_1) &= \left(\frac{\homomatrix\homomatrix^\top\thetastar_2}{\|\homomatrix\homomatrix^\top\thetastar_2\|_2}\right)^\top\Lambda > 0.
\end{aligned}\right.
\end{align*}
\end{example}

To prove \Cref{col:bounded-reduction-m}, we first introduce the following lemma that leads to intermediate results. These results will also be useful later when we show \Cref{apx-col:improved-chance}.

\begin{lemma}
Under the two main assumptions, (H1) and (H2), suppose that Assumptions (M1), (M2), and (M3) hold for all DMs and each DM considers only two choices {$\bm{\theta}_{jt}\in\{\bm{\theta}_j^*,\bm{\theta}_j^{\text{AO}}\}$} for {$j\in[n]$}. Further, let $i$ denote an arbitrary DM and suppose
\begin{enumerate}
    \item $\|\thetastar_i\|_2 \leq 1$;
    \item $\vecalpha_j=(k_j-\gamma_j)\homomatrix\homomatrix^\top\thetastar_j$ with $k_j,\gamma_j>0 \quad \forall j=1,\ldots,n$;
    \item $(\bm{\theta}_{jt})^\top\homomatrix\homomatrix^\top(\thetaao_i-\thetastar_i) \geq 0$ for all decision makers $j\neq i$.
\end{enumerate}
Then, we obtain the following two inequalities for our decision maker $i$ and any decision maker $j$ (where $j\neq i$):
\begin{enumerate}
    \item $(\bm{\theta}_{jt})^\top\homomatrix\homomatrix^\top(\thetaao_i-\thetastar_i) \geq 0$;
    \item $\lambda := \big(\mathbf{a}_t(\bm{\theta}_{\diamond}^\text{all})\big)^\top\homomatrix^\top\thetaao_i - \big(\mathbf{a}_t(\bm{\theta}_{\bullet}^\text{all})\big)^\top\homomatrix^\top\thetastar_i \geq 0$.
\end{enumerate}
\end{lemma}

\begin{proof}
The 2nd condition of this lemma (i.e., on $\vecalpha_j$) implies that $\thetaao_j=\arg\max_{\bm{\theta}_{jt}:\|\bm{\theta}_{jt}\|_2\leq1}\text{cPI}_j(\{\bm{\theta}_{jt},\bm{\theta}_t^{-j}\}) \quad \forall \bm{\theta}_t^{-j}$, from \Cref{col:max-improvement-m}.

Moreover, since this condition also holds for the decision maker $i$, with the 1st condition, i.e., $\|\thetastar_i\|_2 \leq 1$, we now have $\text{cPI}_i(\{\thetastar_i,\bm{\theta}_t^{-i}\}) \leq \text{cPI}_i(\{\thetaao_i,\bm{\theta}_t^{-i}\}) \quad \forall\bm{\theta}_t^{-i}$ and results in the following:
\begin{alignat*}{2}
&\ \text{cPI}_i(\{\thetastar_i,\bm{\theta}_t^{-i}\}) &&\leq \text{cPI}_i(\{\thetaao_i,\bm{\theta}_t^{-i}\})
\\
&\ \Big(\ldots + \gamma_i\thetastar_i\Big)^\top\homomatrix\homomatrix^\top\thetastar_i &&\leq \Big(\ldots+\gamma_i\thetaao_i\Big)^\top\homomatrix\homomatrix^\top\thetastar_i
\\
\Rightarrow &\ \gamma_i(\thetastar_i)^\top\homomatrix\homomatrix^\top\thetastar_i &&\leq \gamma_i(\thetaao_i)^\top\homomatrix\homomatrix^\top\thetastar_i
\\
\Rightarrow &\ (\homomatrix^\top\thetastar_i)^\top(\homomatrix^\top\thetastar_i) &&\leq (\homomatrix^\top\thetaao_i)^\top(\homomatrix^\top\thetastar_i)
,
\end{alignat*}
where we cancel out the common terms from both side via subtraction in the second inequality and then we cancel out $\gamma_i$ as it is non-negative.

Next, by using the inequality $2\mathbf{u}^\top \mathbf{v}\leq\|\mathbf{u}\|^2+\|\mathbf{v}\|^2$ and treating $\homomatrix^\top\thetastar_i$ and $\homomatrix^\top\thetaao_i$ as $\mathbf{u}$ and $\mathbf{v}$, we obtain
\begin{alignat*}{2}
&\ (\homomatrix^\top\thetastar_i)^\top(\homomatrix^\top\thetastar_i) &&\leq (\homomatrix^\top\thetaao_i)^\top(\homomatrix^\top\thetaao_i).
\end{alignat*}

In addition, the above derivation also gives us the following:
\begin{align*}
\left\{
\begin{aligned}
\thetastar_i\homomatrix\homomatrix^\top(\thetaao_i-\thetastar_i)&\geq0,
\\
\thetaao_i\homomatrix\homomatrix^\top(\thetaao_i-\thetastar_i)&\geq0.
\end{aligned}\right.
\end{align*}

For $j\in[n]$, the 2nd condition of this lemma (i.e., on $\vecalpha_j$) leads to $\thetaao_j$ being a normalised vector of $\homomatrix\homomatrix\thetastar_j$ (see \Cref{col:max-improvement-m}). Then, using the 3rd condition of this lemma (i.e., on $\thetastar_j$ and $\thetastar_i$), we then obtain the first result:
\begin{align*}
&\left\{\begin{aligned}
(\thetastar_j)^\top\homomatrix\homomatrix^\top(\thetaao_i-\thetastar_i) &\geq 0 
\\
(\thetaao_j)^\top\homomatrix\homomatrix^\top(\thetaao_i-\thetastar_i) 
&\propto \underbrace{(\homomatrix\homomatrix^\top\thetastar_j)^\top\homomatrix\homomatrix^\top(\thetaao_i-\thetastar_i)}_{\geq 0}
\end{aligned}\right.
\\
&\Rightarrow (\bm{\theta}_{jt})^\top\homomatrix\homomatrix^\top(\thetaao_i-\thetastar_i) \geq 0
.
\end{align*}

This also leads to the 2nd result:
\begin{align*}
\lambda &:= \big(\mathbf{a}_t(\bm{\theta}_{\diamond}^\text{all})\big)^\top\homomatrix^\top\thetaao_i - \big(\mathbf{a}_t(\bm{\theta}_{\bullet}^\text{all})\big)^\top\homomatrix^\top\thetastar_i
\\
&= \big(\sum_{j\neq i}\gamma_j\bm{\theta}_{jt}+\gamma_i\thetaao_i\big)^\top\homomatrix^\top\thetaao_i 
 - \big(\sum_{j\neq i}\gamma_j\bm{\theta}_{jt}+\gamma_i\thetastar_i\big)^\top\homomatrix^\top\thetastar_i
\\
&= \underbrace{\Big(\sum_{j\neq i}\gamma_j\bm{\theta}_{jt}\Big)^\top\homomatrix\homomatrix^\top(\thetaao_i-\thetastar_i)}_{\geq 0}
 + \underbrace{\gamma_i\Big((\thetaao_i)^\top\homomatrix\homomatrix^\top\thetaao_i-(\thetastar_i)^\top\homomatrix\homomatrix^\top\thetastar_i\Big)}_{\geq 0},
\end{align*}
which implies that $\lambda \geq 0$.
\end{proof}

We are now in a position to present the main proof of \Cref{col:bounded-reduction-m}.

\begin{proof}[Main Proof]
We write the proof for an arbitrary decision maker $i$ for whom all conditions in \Cref{col:bounded-reduction-m} hold.
From the condition on Lipschitz continuity, we have the following for an arbitrary agent with the baseline $\B_t$:
\begin{align*}
\left|\xi_i(\bm{\theta}_\bullet^\text{all})-\xi_i(\bm{\theta}_\diamond^\text{all})\right| &= \left|\widetilde{\bm{\delta}}_i(\hat{y}_{it}\big(\bm{\theta}_\bullet^\text{all};\B_t)\big) - \widetilde{\bm{\delta}}_i(\hat{y}_{it}\big(\bm{\theta}_\diamond^\text{all};\B_t)\big)\right| 
\\
&\leq L\left|\hat{y}_{it}\big(\bm{\theta}_\bullet^\text{all};\B_t)\big)-\hat{y}_{it}\big(\bm{\theta}_\diamond^\text{all};\B_t)\big)\right|,
\end{align*}
which implies that
\begin{align*}
\xi_i(\bm{\theta}_\bullet^\text{all})-\xi_i(\bm{\theta}_\diamond^\text{all}) &\leq L\left|\hat{y}_{it}\big(\bm{\theta}_\bullet^\text{all};\B_t)\big)-\hat{y}_{it}\big(\bm{\theta}_\diamond^\text{all};\B_t)\big)\right|.
\end{align*}%
Because $\widetilde{\bm{\delta}}_i(\hat{y}_{it})$ is an increasing function w.r.t. $\hat{y}_{it}$, then if $\hat{y}_{it}\big(\bm{\theta}_\bullet^\text{all};\B_t)\big)<\hat{y}_{it}\big(\bm{\theta}_\diamond^\text{all};\B_t)\big)$, we have
\begin{eqnarray*}
\widetilde{\bm{\delta}}_i(\hat{y}_{it}\big(\bm{\theta}_\bullet^\text{all};\B_t)\big) - \widetilde{\bm{\delta}}_i(\hat{y}_{it}\big(\bm{\theta}_\diamond^\text{all};\B_t)\big) < 0 
&\Rightarrow& \xi_i(\bm{\theta}_\bullet^\text{all})-\xi_i(\bm{\theta}_\diamond^\text{all}) < 0
\\
&\Rightarrow& p\big(\xi_i(\bm{\theta}_\bullet^\text{all})-\xi_i(\bm{\theta}_\diamond^\text{all})>M\big) = 0
.
\end{eqnarray*}

Next, we show the proof for the case where there is a reduction in an agent's prediction, i.e., $\hat{y}_i\big(\bm{\theta}_\bullet^\text{all};\B_t)\big)>\hat{y}_i\big(\bm{\theta}_\diamond^\text{all};\B_t)\big)$. It leads to the following:
\begin{align*}
\xi_i(\bm{\theta}_\bullet^\text{all})-\xi_i(\bm{\theta}_\diamond^\text{all}) &\leq L\left|\hat{y}_{it}\big(\bm{\theta}_\bullet^\text{all};\B_t)\big)-\hat{y}_{it}\big(\bm{\theta}_\diamond^\text{all};\B_t)\big)\right|
\\
&= L\left(\hat{y}_{it}\big(\bm{\theta}_\bullet^\text{all};\B_t)\big)-\hat{y}_{it}\big(\bm{\theta}_\diamond^\text{all};\B_t)\big)\right).
\end{align*}

When $\B_t\sim\mathcal{N}(0,\sigma^2I)$, we have the following:
\begin{eqnarray*}
\operatorname{Var}\big(\B_t^\top(\thetaao_i-\thetastar_i)\big) = \|\thetaao_i-\thetastar_i\|_2^2\cdot\sigma^2
&\Rightarrow& \B_t^\top(\thetaao_i-\thetastar_i)\sim\mathcal{N}(0,\ \|\thetaao_i-\thetastar_i\|_2^2\cdot\sigma^2)
\\
&\Rightarrow &\frac{\B_t^\top(\thetaao_i-\thetastar_i)}{\|\thetaao_i-\thetastar_i\|_2\cdot\sigma}\sim\mathcal{N}(0,1).
\end{eqnarray*}
We use $\mathbf{a}_t(\bm{\theta}_\diamond^\text{all})$ and $\mathbf{a}_t(\bm{\theta}_\bullet^\text{all})$ to denote the best responses of an agent $t$ with respect to $\bm{\theta}_\diamond^\text{all}$ and $\bm{\theta}_\bullet^\text{all}$, where
\begin{align*}
\mathbf{a}_t(\bm{\theta}_\diamond^\text{all}) &= \homomatrix^\top\Big(\sum_{j\neq i}\gamma_j\bm{\theta}_{jt}+\gamma_i\thetaao_i\Big),
\\
\mathbf{a}_t(\bm{\theta}_\bullet^\text{all}) &= \homomatrix^\top\Big(\sum_{j\neq i}\gamma_j\bm{\theta}_{jt}+\gamma_i\thetastar_i\Big)
.
\end{align*}

Then, the counterfactual change in the predicted outcome $\hat{y}_{it}$ of an arbitrary agent at round $t$ is
\begin{align*}
\hat{y}_{it}(\bm{\theta}_\diamond^\text{all};\B_t)-\hat{y}_{it}(\bm{\theta}_\bullet^\text{all};\B_t)
=\ &\big(\B_t+\homomatrix\mathbf{a}_t(\bm{\theta}_\diamond^\text{all})\big)^\top\thetaao_i - \big(\B_t+\homomatrix\mathbf{a}_t(\bm{\theta}_{\bullet}^\text{all})\big)^\top\thetastar_i
\\
=\ &\B_t^\top(\thetaao_i-\thetastar_i) + \underbrace{\big(\mathbf{a}_t(\bm{\theta}_\diamond^\text{all})\big)^\top\homomatrix^\top\thetaao_i - \big(\mathbf{a}_t(\bm{\theta}_{\bullet}^\text{all})\big)^\top\homomatrix^\top\thetastar_i}_{\geq0}
\\
=\ &\B_t^\top(\thetaao_i-\thetastar_i) + \lambda
.
\end{align*}%

Then, the counterfactual reduction in admission chance is
\begin{align*}
\xi_i(\bm{\theta}_\bullet^\text{all})-\xi_i(\bm{\theta}_\diamond^\text{all}) &\leq L\left(\hat{y}_{it}\big(\bm{\theta}_\bullet^\text{all};\B_t)\big)-\hat{y}_{it}\big(\bm{\theta}_\diamond^\text{all};\B_t)\big)\right)
\\
&= L\left(\B_t^\top(\thetastar_i-\thetaao_i) - \lambda\right)
.
\end{align*}

Next, for any $M>0$, it follows that
\begin{align*}
p\big(\xi_i(\bm{\theta}_\bullet^\text{all})-\xi_i(\bm{\theta}_\diamond^\text{all})>M\big) 
&\leq p\big(L\left(\B_t^\top(\thetastar_i-\thetaao_i) - \lambda\right)>M\big)
\\
&= p\big(\B_t^\top(\thetastar_i-\thetaao_i)>M/L+\lambda\big)
\\
&= p\left(\frac{\B_t^\top(\thetastar_i-\thetaao_i)}{\|\thetaao_i-\thetastar_i\|_2\cdot\sigma}>\frac{M/L+\lambda}{\|\thetaao_i-\thetastar_i\|_2\cdot\sigma}\right)
\\
&= 1-\operatorname{CDF}\left(\frac{M/L+\lambda}{\|\thetaao_i-\thetastar_i\|_2\cdot\sigma}\right)
\\
&= 1-\Phi\left(\frac{M/L+\lambda}{\|\thetaao_i-\thetastar_i\|_2\cdot\sigma}\right)
\\
&= \Phi\left(\frac{-M/L-\lambda}{\|\thetaao_i-\thetastar_i\|_2\cdot\sigma}\right)
.
\end{align*}

This concludes the proof.
\end{proof}

We now introduce the corollary on the admission chance of agents into other environments $j\neq i$ that we briefly mentioned in the main paper.

\begin{corollary}[Improved chance]\label{apx-col:improved-chance}
Suppose assumptions (H1), (H2), (M1)-(M3) hold for all DMs and each DM considers only two choices {$\bm{\theta}_{jt}\in\{\bm{\theta}_j^*,\bm{\theta}_j^{\text{AO}}\}$} for {$j\in[n]$}. Further, let $i$ be an arbitrary DM and suppose
\begin{enumerate}[label={(\arabic*)}]
    \item {$\|\bm{\theta}_i^*\|_2 \leq 1$}; 
    \item {$\vecalpha_j=(k_j-\gamma_j)\homomatrix\homomatrix^\top\bm{\theta}_j^*$} with {$k_j,\gamma_j>0$} for {$j\in[n]$};
    \item {$(\bm{\theta}_{jt})^\top\homomatrix\homomatrix^\top(\thetaao_i-\thetastar_i) \geq 0$} for {$j\neq i$};
    \item For $j\neq i$, each selection function, {$\bm{\delta}_j(\mathbf{x}; \bm{\theta}_{jt}):=\widetilde{\bm{\delta}}_j(\hat{y}_{jt}(\bm{\theta}_t^\text{all}))$}, is increasing w.r.t. $\hat{y}_{jt}$.
\end{enumerate}
Then,
$$p\big(\xi_i(\bm{\theta}^\text{all}_{\bullet}) \leq \xi_i(\bm{\theta}^\text{all}_{\diamond})\big) = 1,$$
where {$\bm{\theta}^\text{all}_{\bullet} = \{\bm{\theta}_t^{-i},\thetastar_i\}$}, {$\bm{\theta}^\text{all}_{\diamond} = \{\bm{\theta}_t^{-i},\thetaao_i\}$}, and {$\xi_j(\bm{\theta}_t^\text{all}) := p(W_{jt}=1|\B_t;\bm{\theta}_t^\text{all})$} denotes the admission chance of the agent $t$ from the DM $j$, given released decision parameters {$\bm{\theta}_t^\text{all}$}.
\end{corollary}

\begin{proof}
The counterfactual change in the predicted outcome $\hat{y}_{jt}$ (for any $j\neq i$) of the agent $t$ is
\begin{align*}
\hat{y}_{jt}(\bm{\theta}_\diamond^\text{all};\B_t)-\hat{y}_{jt}(\bm{\theta}_\bullet^\text{all};\B_t)
=\ &\big(\B_t+\homomatrix\mathbf{a}_t(\bm{\theta}_{\diamond}^\text{all})\big)^\top\bm{\theta}_{jt} - \big(\B_t+\homomatrix\mathbf{a}_t(\bm{\theta}_{\bullet}^\text{all})\big)^\top\bm{\theta}_{jt}
\\
=\ &(\mathbf{a}_\diamond^\top-\mathbf{a}_\bullet^\top)\homomatrix^\top\bm{\theta}_{jt}
\\
=\ &\gamma_i(\thetaao_i-\thetastar_i)^\top\homomatrix\homomatrix^\top\bm{\theta}_{jt}
\\
\geq\ &0
.
\end{align*}

When the admission rates are increasing functions of the predicted performance, we obtain the desired result:
$$p(W_{jt}=1\mid\B_t;\ \bm{\theta}_\diamond^\text{all}) \geq p(W_{jt}=1\mid\B_t;\ \bm{\theta}_\bullet^\text{all}) \quad \forall j\neq i.$$

This concludes the proof.
\end{proof}


\section{On Estimation of \texorpdfstring{$\homomatrix\homomatrix^\top$}{EET} and \texorpdfstring{$\bm{\theta}_i^\text{AO}$}{thetaAO}}

In this section, we show how OLS can be used to estimate $\homomatrix\homomatrix^\top$ and $\thetaao_i$. Even though our MSLR algorithm can be trivially adapted to estimate $\homomatrix\homomatrix^\top\thetastar_i$ directly, having $\homomatrix\homomatrix^\top$ is useful when one is already given $\thetastar_i$ and wants to compute $\homomatrix\homomatrix^\top\thetastar_i$, thus avoiding the hassle of deploying $\bm{\theta}_i$ in specific manners as outlined in \Cref{theorem:local-exo-s} and \Cref{def:protocol}.

To ensure consistency with the notations of our setup in \Cref{sec:prob-formulation}, in this subsection, we use $T$ to denote the number of observations (previously \textit{rounds}). We restate the definition for the OLS estimator \citep{hastie2017elements} below.

\begin{definition}[OLS Estimator]\label{def:ols}
Let $\mathbf{Y}:=\mathbf{X}\mathbf{B}+\mathbf{E}$ be the true data generating mechanism, where $\mathbf{Y}\in\mathbb{R}^{T\times K}$ is a matrix containing $T$ observations of $K$ target variables, $\mathbf{X}\in\mathbb{R}^{T\times p}$ is the data matrix containing $T$ observations of the $p$-dimensional covariates, $\mathbf{B}\in\mathbb{R}^{p\times K}$ is matrix of coefficients, and $\mathbf{E}\in\mathbb{R}^{T\times K}$ is a matrix of the observed noise. Then the OLS estimator for $\mathbf{B}$ that minimises the residual sum-of-squares is given as
$$\widehat{\mathbf{B}} = (\mathbf{X}^\top\mathbf{X})^{-1}\mathbf{X}^\top\mathbf{Y}.$$
\end{definition}

For clarity, we first introduce the following two lemmas that will later be used in defining our OLS estimators for $\homomatrix\homomatrix^\top$ and $\thetaao_i$.

\begin{lemma}
\label{lemma:generic-ols-with-biases}
Let $\mathbf{Y}:=\mathbf{X}\mathbf{B}+(\vec{1}_T)\mathbf{c}^\top+\mathbf{E}$ be the true data generating mechanism, where $\mathbf{c}\in\mathbb{R}^K$ is a column vector containing the biases associated with $K$ target variables, $\vec{1}_T$ denotes a column vector containing $T$ values of $1$, and the rest of the terms conform to our standard OLS setup (\Cref{def:ols}), then the OLS estimator for $\mathbf{B}$ and $\mathbf{c}$ is
\begin{align*}
\begin{bmatrix}
    \widehat{\mathbf{B}} \\
    \widehat{\mathbf{c}}^\top
\end{bmatrix}
=
\begin{bmatrix}
    \mathbf{X}^\top\mathbf{X} & \mathbf{X}^\top(\vec{1}_T) \\
    (\vec{1}_T)^\top\mathbf{X} & T
\end{bmatrix}^{-1}
\begin{bmatrix}
    \mathbf{X}^\top\mathbf{Y} \\
    (\vec{1}_T)^\top\mathbf{Y}
\end{bmatrix}
.
\end{align*}
\end{lemma}

\begin{proof}
Rewrite the function of $\mathbf{Y}$ in terms of block matrices:
\begin{align*}
\mathbf{Y} &:= 
\begin{bmatrix}
    \mathbf{X} & \vec{1}_T
\end{bmatrix}
\begin{bmatrix}
    \mathbf{B} \\
    \mathbf{c}^\top
\end{bmatrix}
,
\\
\mathbf{Y} &:= \mathbf{X}^\prime \mathbf{B}^\prime
,
\end{align*}
where we denote the two matrices on the right as $\mathbf{X}^\prime$ and $\mathbf{B}^\prime$ respectively.
Then from \Cref{def:ols}, the OLS estimator for $\mathbf{B}^\prime$ is
\begin{align*}
\widehat{\mathbf{B}}^\prime =  \big((\mathbf{X}^\prime)^\top(\mathbf{X}^\prime)\big)^{-1}(\mathbf{X}^\prime)^\top\mathbf{Y}
\quad\Rightarrow\quad
\begin{bmatrix}
    \widehat{\mathbf{B}} \\
    \widehat{\mathbf{c}}^\top
\end{bmatrix}
&=
\Bigg(
\begin{bmatrix}
    \mathbf{X}^\top \\
    (\vec{1}_T)^\top
\end{bmatrix}
\begin{bmatrix}
    \mathbf{X} & \vec{1}_T
\end{bmatrix}
\Bigg)^{-1}
\begin{bmatrix}
    \mathbf{X}^\top \\
    (\vec{1}_T)^\top
\end{bmatrix}
\mathbf{Y}
\\
\begin{bmatrix}
    \widehat{\mathbf{B}} \\
    \widehat{\mathbf{c}}^\top
\end{bmatrix}
&=
\begin{bmatrix}
    \mathbf{X}^\top\mathbf{X} & \mathbf{X}^\top(\vec{1}_T) \\
    (\vec{1}_T)^\top\mathbf{X} & T
\end{bmatrix}^{-1}
\begin{bmatrix}
    \mathbf{X}^\top\mathbf{Y} \\
    (\vec{1}_T)^\top\mathbf{Y}
\end{bmatrix}
.
\end{align*}
This concludes the proof.
\end{proof}

\begin{lemma}
\label{lemma:generic-ols-the-harris-way}
With the same setup in the previous lemma, let us rewrite $\mathbf{X}$ and $\mathbf{Y}$ as columns of row vectors $\big[\mathbf{x}_1,\ldots,\mathbf{x}_T\big]^\top$ and $\big[\mathbf{y}_1,\ldots,\mathbf{y}_T\big]^\top$, in which each $\mathbf{x}_t\in\mathbb{R}^p$ and $\mathbf{y}_t\in\mathbb{R}^K$. Furthermore, we denote $\widetilde{\mathbf{x}}_t=\big[\mathbf{x}_t^\top \quad 1\big]^\top$, then the OLS estimator can be rewritten as
\begin{align*}
\begin{bmatrix}
    \widehat{\mathbf{B}} \\
    \widehat{\mathbf{c}}^\top
\end{bmatrix}
=
\left(\sum_{t=1}^T\widetilde{\mathbf{x}}_t\widetilde{\mathbf{x}}_t^\top\right)^{-1} \sum_{t=1}^T\widetilde{\mathbf{x}}_t\mathbf{y}_t^\top
.
\end{align*}
\end{lemma}

\begin{proof}
Let $\widetilde{\mathbf{X}}=\big[\widetilde{\mathbf{x}}_1,\ldots,\widetilde{\mathbf{x}}_T\big]^\top$ then following \Cref{lemma:generic-ols-with-biases}, we have
\begin{align*}
\widehat{\mathbf{B}}^\prime =  \big((\mathbf{X}^\prime)^\top(\mathbf{X}^\prime)\big)^{-1}(\mathbf{X}^\prime)^\top\mathbf{Y}
\quad\Rightarrow\quad
\begin{bmatrix}
    \widehat{\mathbf{B}} \\
    \widehat{\mathbf{c}}^\top
\end{bmatrix}
&=
\Bigg(
\begin{bmatrix}
    \mathbf{X}^\top \\
    (\vec{1}_T)^\top
\end{bmatrix}
\begin{bmatrix}
    \mathbf{X} & \vec{1}_T
\end{bmatrix}
\Bigg)^{-1}
\begin{bmatrix}
    \mathbf{X}^\top \\
    (\vec{1}_T)^\top
\end{bmatrix}
\mathbf{Y}
\\
&=
\big(\widetilde{\mathbf{X}}^\top\widetilde{\mathbf{X}}\big)^{-1}\widetilde{\mathbf{X}}^\top\mathbf{Y}
\\
&=
\big(\widetilde{\mathbf{X}}^\top\widetilde{\mathbf{X}}\big)^{-1}
\begin{bmatrix}
    \widetilde{\mathbf{x}}_1 & \ldots & \widetilde{\mathbf{x}}_T
\end{bmatrix}
\begin{bmatrix}
    \mathbf{y}_1^\top \\
    \vdots \\
    \mathbf{y}_T^\top
\end{bmatrix}
\\
&= \left(\sum_{t=1}^T\widetilde{\mathbf{x}}_t\widetilde{\mathbf{x}}_t^\top\right)^{-1} \sum_{t=1}^T\widetilde{\mathbf{x}}_t\mathbf{y}_t^\top
.
\end{align*}
This concludes the proof.
\end{proof}

We are now in a position to present the way to estimate $\mathbb{E}\left[\mathcal{E}_t\mathcal{E}_t^\top\right]$, which is equivalent to $\homomatrix\homomatrix^\top$ in our setting.

\begin{lemma}
Let $\widetilde{\Theta}_t=\big[\sum_{i=1}^n\gamma_{it}\bm{\theta}_{it}^\top \quad 1\big]^\top$ and $\Omega=\mathbb{E}\left[\mathcal{E}_t\mathcal{E}_t^\top\right]$, the OLS estimator for $\Omega$ and $\mathbb{E}\left[\B_t\right]$ is
\begin{align*}
\begin{bmatrix}
    \widehat{\Omega} \\
    \bar{\mathbf{b}}^\top
\end{bmatrix}
=
\left(\sum_{t=1}^T\widetilde{\Theta}_t\widetilde{\Theta}_t^\top\right)^{-1}\sum_{t=1}^T\widetilde{\Theta}_t\mathbb{E}\left[\X_t^\top\mid\bm{\theta}_t^\text{all}\right]
.
\end{align*}
\end{lemma}

\begin{proof}
Note that this result is similar to that of \citet{harris2022strategic} except that we provide this for the setting of multiple decision makers and will prove this using \Cref{lemma:generic-ols-the-harris-way}.

Rewriting the structural causal function of $\X_t$ under the best response $\mathbf{a}_t$ as follows:
\begin{align*}
\X_t &= \B_t + \mathcal{E}_t\mathcal{E}_t^\top\left(\sum_{i=1}^n\gamma_{it}\bm{\theta}_{it}\right)
\quad\Rightarrow\quad
\mathbb{E}\left[\X_t\mid\bm{\theta}_t^\text{all}\right] = \mathbb{E}\left[\B_t\right] + \mathbb{E}\left[\mathcal{E}_t\mathcal{E}_t^\top\right]\left(\sum_{i=1}^n\gamma_{it}\bm{\theta}_{it}\right),
\end{align*}
where the two conditional expectation terms on the left-hand side simplify because $\B_t$ and $\mathcal{E}_t$ are marginally independent of $\bm{\theta}_t^\text{all}$. Then transposing both sides, we obtain an equation for each data point $\bm{\theta}_t^\text{all}$:
\begin{align*}
\mathbb{E}\left[\X_t^\top\mid\bm{\theta}_t^\text{all}\right] &= \mathbb{E}\left[\B_t^\top\right] + \left(\sum_{i=1}^n\gamma_{it}\bm{\theta}_{it}^\top\right)\mathbb{E}\left[\mathcal{E}_t\mathcal{E}_t^\top\right]
\\
&=
\begin{bmatrix}
    \sum_{i=1}^n\gamma_{it}\bm{\theta}_{it}^\top & 1
\end{bmatrix}
\begin{bmatrix}
    \mathbb{E}\left[\mathcal{E}_t\mathcal{E}_t^\top\right] \\
    \mathbb{E}\left[\B_t^\top\right]
\end{bmatrix}
\\
&= \widetilde{\Theta}_t^\top 
\begin{bmatrix}
    \mathbb{E}\left[\mathcal{E}_t\mathcal{E}_t^\top\right] \\
    \mathbb{E}\left[\B_t^\top\right]
\end{bmatrix}
.
\end{align*}

Stacking all data points vertically, let $\mathbf{X}^\top_\mathbb{E}=\Big[\mathbb{E}\left[\X_1\mid\bm{\theta}_1^\text{all}\right],\ldots,\mathbb{E}\left[\X_T\mid\bm{\theta}_T^\text{all}\right]\Big]^\top$ and $\widetilde{\bm{\Theta}}=\Big[\widetilde{\Theta}_1,\ldots,\widetilde{\Theta}_T\Big]^\top$, we have the following linear setup that results in an OLS estimator:
\begin{align*}
\mathbf{X}^\top_\mathbb{E} = \widetilde{\bm{\Theta}}
\begin{bmatrix}
    \mathbb{E}\left[\mathcal{E}_t\mathcal{E}_t^\top\right] \\
    \mathbb{E}\left[\B_t^\top\right]
\end{bmatrix}
\quad\Rightarrow\quad 
\begin{bmatrix}
    \widehat{\Omega} \\
    \bar{\mathbf{b}}^\top
\end{bmatrix} &= \big(\widetilde{\bm{\Theta}}^\top\widetilde{\bm{\Theta}}\big)^{-1}\widetilde{\bm{\Theta}}^\top\mathbf{X}^\top_\mathbb{E}
\\
&= \big(\widetilde{\bm{\Theta}}^\top\widetilde{\bm{\Theta}}\big)^{-1}
\begin{bmatrix}
    \widetilde{\Theta}_1 & \ldots & \widetilde{\Theta}_T
\end{bmatrix}
\begin{bmatrix}
    \mathbb{E}\left[\X_1^\top\mid\bm{\theta}_1^\text{all}\right] \\
    \vdots \\
    \mathbb{E}\left[\X_T^\top\mid\bm{\theta}_T^\text{all}\right]
\end{bmatrix}
\\
&=
\left(\sum_{t=1}^T\widetilde{\Theta}_t\widetilde{\Theta}_t^\top\right)^{-1}\sum_{t=1}^T\widetilde{\Theta}_t\mathbb{E}\left[\X_t^\top\mid\bm{\theta}_t^\text{all}\right]
.
\end{align*}
This concludes the proof.
\end{proof}

We present the way to estimate $\thetaao_i$ using the following theorem, where $\bm{\kappa}_i$ denotes the un-normalised version of $\thetaao_i$.

\begin{lemma}
Suppose that Assumptions (M1), (M2), and (M3) hold for an arbitrary decision maker $i$ and denote the following:
\begin{equation*}
\bm{\kappa}_i = \vecalpha_i+\gamma_i\homomatrix\homomatrix^\top\bm{\theta}_i^*, \quad
\zeta_i(\bm{\theta}_\diamond^{-i}) = \beta_i + \psi_i(\bm{\theta}_\diamond^{-i}), \quad
\widetilde{\Theta}_{it_s} = [\bm{\theta}_{it_s}^\top \quad 1]^\top,
\end{equation*}
where we constrain our data set to the case $\bm{\theta}_t^{-i}=\bm{\theta}_\diamond^{-i}$ for some arbitrary value $\bm{\theta}_\diamond^{-i}$ and we let the samples be indexed by $t_s=1,\ldots,T_s$ with $T_s$ being the sample size, then the OLS estimator for $\bm{\kappa}_i$ and $\zeta_i(\bm{\theta}_\diamond^{-i})$ is
\begin{align*}
\begin{bmatrix}
    \widehat{\bm{\kappa}}_i \\
    \widehat{\zeta}_i(\bm{\theta}_\diamond^{-i})
\end{bmatrix}
=
\left(\sum_{t_s=1}^{T_s}\widetilde{\Theta}_{it_s}\widetilde{\Theta}_{it_s}^\top\right)^{-1}\sum_{t_s=1}^{T_s}\widetilde{\Theta}_{it_s}\mathcal{Q}_i(\{\bm{\theta}_{it_s},\bm{\theta}_\diamond^{-i}\})
.
\end{align*}%
\end{lemma}

\begin{proof}
Recall the decomposition of $\mathcal{Q}_i$ from the proof for \Cref{theorem:bounded-optim-m},
\begin{align*}
\mathcal{Q}_i(\{\bm{\theta}_{it},\bm{\theta}_t^{-i}\}) 
&= \big(\vecalpha_i+\gamma_i\homomatrix\homomatrix^\top\bm{\theta}_i^*\big)^\top\bm{\theta}_{it} + \Big(\beta_i + \psi_i(\bm{\theta}_t^{-i})\Big)
\\
&=
\bm{\kappa}_i^\top\bm{\theta}_{it} + \zeta_i(\bm{\theta}_t^{-i})
\\
&=
\begin{bmatrix}
    \bm{\theta}_{it}^\top & 1
\end{bmatrix}
\begin{bmatrix}
    \bm{\kappa}_i \\
    \zeta_i(\bm{\theta}_t^{-i})
\end{bmatrix}
\end{align*}
where we could transpose the terms in the last equation because $\mathcal{Q}_i(\{\bm{\theta}_{it},\bm{\theta}_t^{-i}\})$ is a scalar.

Note that neither the value of $\bm{\theta}_{it}$ nor that of $\bm{\theta}_t^{-i}$ affects $\bm{\kappa}_i$ and the functional form of $\zeta_i$, due to Assumption (M2). When we restrict ourselves to the data set where $\bm{\theta}_t^{-i}=\bm{\theta}_\diamond^{-i}$ for some arbitrary value $\bm{\theta}_\diamond^{-i}$ then we obtain the desired result, following \Cref{lemma:generic-ols-the-harris-way} on OLS estimator,
\begin{align*}
\mathcal{Q}_i(\{\bm{\theta}_{it_s},\bm{\theta}_\diamond^{-i}\})
=
\begin{bmatrix}
    \bm{\theta}_{it_s}^\top & 1
\end{bmatrix}
\begin{bmatrix}
    \bm{\kappa}_i \\
    \zeta_i(\bm{\theta}_\diamond^{-i})
\end{bmatrix}
.
\end{align*}
This concludes the proof.
\end{proof}

\section{Detailed Simulation Setup and Additional Experiments} \label{apx:experiments}

\begin{figure}[t!]
    \centering
    \includegraphics[width=0.7\columnwidth]{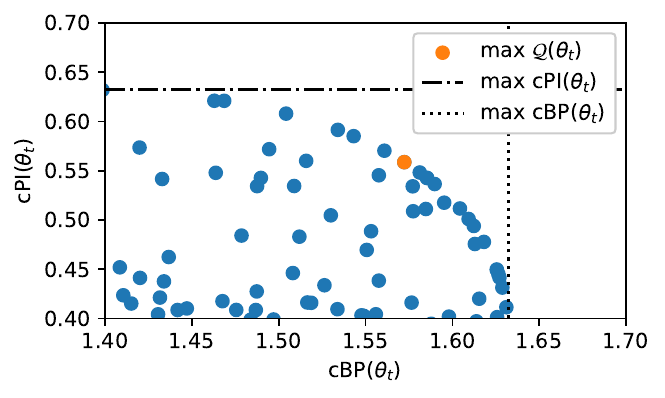}
    \caption{A scatter plot showing different pairs of values for cPI$(\bm{\theta}_t)$ and cBP$(\bm{\theta}_t)$ by varying $\bm{\theta}_{t}$. The pair that maximises $\mathcal{Q}(\bm{\theta}_{t})$ is shown in orange. Dashed and dotted lines show the respective maximum values of cPI and cBP. In this case, cPI$(\bm{\theta}_t)$ and cBP$(\bm{\theta}_t)$ cannot be maximised simultaneously.}
    \label{fig:trade-off}
\end{figure}

\begin{figure*}[t!]
    \centering
    \begin{subfigure}[b]{0.49\textwidth}
        \centering
        \includegraphics[width=\textwidth]{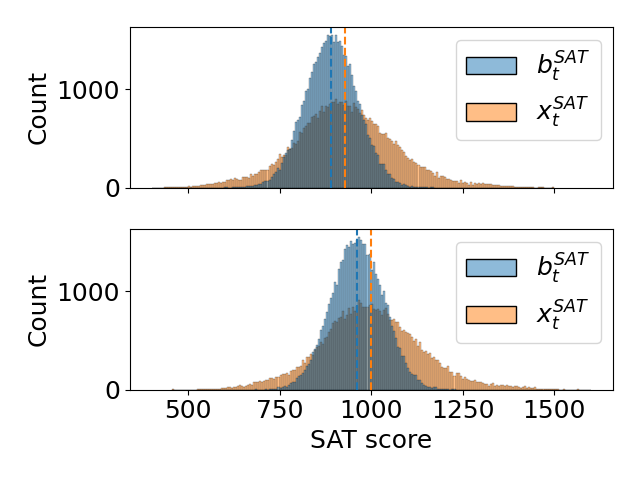}
        \caption{Distribution of $\X_t^\text{SAT}$ among disadvantaged (top) and advantaged (bottom) students. Color-coded vertical lines show the mean of the distribution color-coded by the respective color.
        }
        \label{fig:our-settings-data1}
    \end{subfigure}
    \hfill
    \begin{subfigure}[b]{0.49\textwidth}
        \centering
        \includegraphics[width=\textwidth]{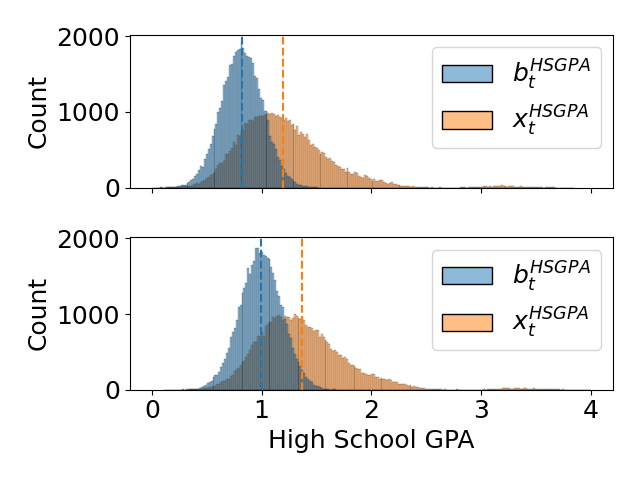}
        \caption{Distribution of $\X_t^\text{HS GPA}$ among disadvantaged (top) and advantaged (bottom) students. Color-coded vertical lines show the mean of the distribution color-coded by the respective color.
        }
        \label{fig:our-settings-data2}
    \end{subfigure}
    \hfill
    \begin{subfigure}[b]{\textwidth}
        \centering
        \includegraphics[width=0.6\textwidth]{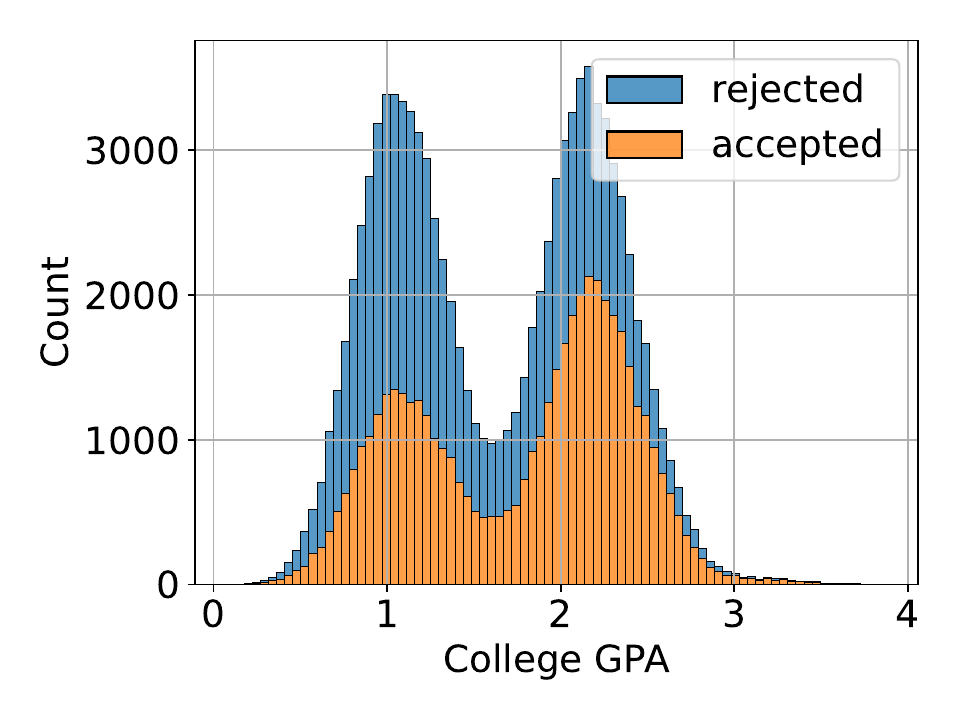}
        \caption{Illustration via stacked bar-plot that in our setup, the distributions of $Y_t$ and $Y_t | W_t = 1$ are not identical. $p(Y_t | W_t=1)$ (orange) is skewed to left, unlike the marginal distribution $p(Y_t)$ (blue and orange combined). }
        \label{fig:our-settings-data3}
    \end{subfigure}
    \caption{Distribution of $\X_t$, $\B_t$ and $Y_{it}$ in our synthetic dataset, for $n=1$. We refer the readers to the text for details.}
    \label{fig:our-settings-data}
\end{figure*}

\begin{figure}[t!]
    \centering
    \begin{subfigure}[t]{0.49\textwidth}
        \centering
        \includegraphics[width=0.95\textwidth]{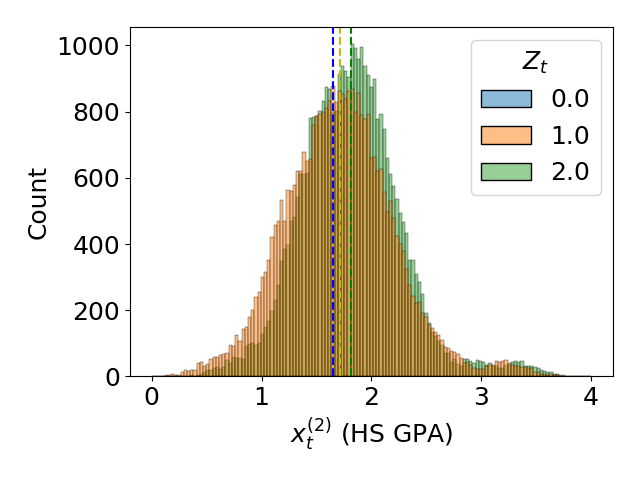}
        \caption{Distribution of High School GPA enrolled in different environments.}
        \label{fig:our-settings-multi2}
    \end{subfigure}
    \hfill
    \begin{subfigure}[t]{0.49\textwidth}
        \centering
        \includegraphics[width=0.95\textwidth]{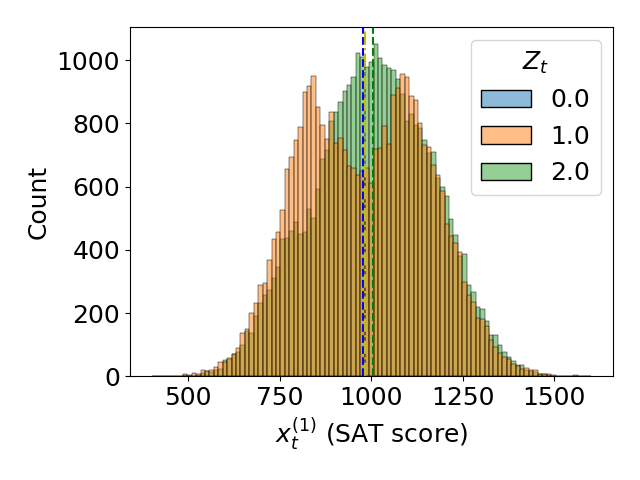}
        \caption{Distribution of SAT score in different environments. }
        \label{fig:our-settings-multi1}
    \end{subfigure}
    \caption{Illustration of heterogeneity in environments. Distribution of $X_t$ in our synthetic dataset, for $n=2$. Color-coded vertical lines show the mean of the distribution color-coded by the respective color. High School GPA is higher for applicants in $Z_t=2$ than in $Z_t=1$ (Figure \ref{fig:our-settings-multi2}). The same applies for the SAT score (Figure \ref{fig:our-settings-multi1}).}
    \label{fig:our-settings-multi}
\end{figure}

\subsection{Our Simulation Setup} \label{apx:exp-simulation-setup}

Following \citet{harris2022strategic}, we consider the problem of predicting the the college GPA (i.e., target) from high school GPA and SAT score (i.e., the covariates. In particular, we denote by $Y_{it} \in \mathbb{R}$ the college GPA of $i$th student in round $t$ in environment $i$. Similarly, let $\X_t = (X_t^{\text{SAT}}, X_t^{\text{HS GPA}})^\top$ denote the covariates of the student in round $t$. Firstly, to create confounding between $\B_t$ and $O_{it}$, we consider two groups: disadvantaged ($g_t=0$), and advantaged ($g_t=1$). We now generate $\B_t= (\B_t^{\text{SAT}}, \B_t^{\text{HS GPA}})^T$ and $O_{it}$ as in \citet{harris2022strategic}:  
\begin{align*}
    \B_t^{\text{SAT}} &\sim \begin{cases} \mathcal{N}(800, 200) & g_t = 0  \\ \mathcal{N}(1000, 200) & g_t = 1\end{cases} \\ 
    \B_t^{\text{HS GPA}} &\sim \begin{cases} \mathcal{N}(1.8, 0.5) & g_t = 0  \\ \mathcal{N}(2.2, 0.5) & g_t = 1\end{cases} \\
    O_{it} &\sim \begin{cases}\mathcal{N}(0.5, 0.2) & g_t = 0 \\ \mathcal{N}(1.5, 0.2) & g_t = 1\end{cases} 
\end{align*}

For simplicity, we let each environment to have equal preference $\gamma_i = \frac{1}{n}$ for $i\in[n]$, where we recall that $n$ denotes the number of environments. We set the effort conversion matrix $\homomatrix$ for all students according to $\mathbb{E}[\mathcal{E}_t]$ in \citet{harris2022strategic}:
\begin{align*}
    \homomatrix = \begin{pmatrix}
        10 & 0 \\ 0 & 1
    \end{pmatrix}
\end{align*}

We then compute $\mathbf{x}_t = \mathbf{b}_t + \homomatrix \mathbf{a}_t$ where $\mathbf{a}_t = \homomatrix^\top \sum_{i=1}^{n} \gamma_{it} \bm{\theta}_{it}$ is the optimal action. To retain the real-world interpretation, we normalize High School GPA and SAT scores to lie between (0, 4) and (400, 1600) respectively. Figure \ref{fig:our-settings-data1} and \ref{fig:our-settings-data2} show the distribution of $\X_t$ and $\B_t$.

Assessment rules $\bm{\theta}_{it}$ are generated heterogeneously such that each environment $i$ emphasizes the HS GPA for prediction more than all environments $j < i$:  
\begin{align*}
    \bm{\theta}_{it} \sim \mathcal{N}\left(\begin{pmatrix}
        1 \\ i
    \end{pmatrix}, \begin{pmatrix}
        10 & \\ & 1
    \end{pmatrix}\right) \qquad i=1, \ldots, n
\end{align*}

Furthermore, as cooperative protocol (\Cref{def:protocol}) requires $\bm{\theta}_{it} = k \bm{\theta}_{it^\prime}$ for two rounds $t, t^\prime$, we generate scaled duplicates of each $\bm{\theta}_{it}$ for each environment. In particular, $\text{for each }\;i, \;\exists \;t, t^\prime \;\text{s.t.}\; \bm{\theta}_{it} = k \bm{\theta}_{it^\prime}$. Finally, we parameterize how often to deploy scaled duplicates by parameters $\{\eta_i: i \in \{1, \ldots, n\}\}$, which we also discuss in \Cref{sec:experiments}.

Afterwards, we compute $w_{it} \in \{0, 1\}$ by selecting $\rho \in [0,1]$ fraction of the students having the highest prediction $\X_{t}^\top \bm{\theta}_{it}$ in environment $i$ in round $t$. Formally, we have
\begin{equation*}
    w_{it} = \begin{cases}
        1 & \text{if $p(\X_{t}^\top\bm{\theta}_{it} < \mathbf{x}_t^\top\bm{\theta}_{it}) > (1 - \rho$)} \\
        0 & \text{o.w.}.
    \end{cases}
\end{equation*} 

We remark that we use the same selection parameters $\bm{\theta}_{it}$ for all students within a round $t$. 

Afterwards, we compute $z_t \in \{0, \ldots, n\}$, which denotes which college the student in round $t$ enrolls in. If $w_{it} = 0 \;\forall i$, then $z_t=0$ (i.e. corresponds to student $t$ being rejected). Otherwise, $z_t$ is randomly sampled from a categorical distribution $Cat(\{i: W_{it} = 1\})$. Event probabilities used for sampling are set to respective normalized preferences $\{\gamma_i: w_{it}=1\}$. We recall that when $n=1$, then $Z_t := W_{t}$.

Finally, we compute $y_{it} = \mathbf{x}_t^\top\bm{\theta}_i^* + o_{it}$, if $z_t = i$, where $\bm{\theta}_i^* = (\theta_{i}^{*,\text{SAT}}, \theta_{i}^{*,\text{HS GPA}})^\top$ is the causal coefficient of environment $i$. As \citet{harris2022strategic} deduced $\thetastar \approx (0, 0.5)$ from a real world dataset, we set $\thetastar_{i} = (\theta_{i}^{*,\text{SAT}}, \theta_{i}^{*,\text{HS GPA}})^\top$ to
\begin{equation*}
    \theta_{i}^{*,\text{SAT}} = 0 \qquad \theta_{i}^{*, \text{HS GPA}} \sim \mathcal{N}(0.5, 0.1). 
\end{equation*}

In our experiments. For $n=1$, Figure \ref{fig:our-settings-data3} illustrates the difference between the distribution of (i) college GPA of selected students (i.e., $Y_{it} | W_{it} = 1$) and (ii) of all students, had they been selected $Y_{it}$. For $n=2$, Figure \ref{fig:our-settings-multi2} shows that, as expected, High School GPA and SAT scores in $Z_t=2$ has a higher mean than in $Z_t=1$, as shown in Figure \ref{fig:our-settings-multi}.


\subsection{Analyses of Assumptions}
\label{apx:analyses-of-assumptions}
In the single DM setting, we test the sensitivity of \Cref{theorem:bounded-optim-s} and \Cref{theorem:local-exo-s} against the linearity assumption on the relationship between $B_t$ and $X_t$, and against Assumption H1. Specifically, to break the linearity between $B_t$ and $X_t$, we apply the standard logistic function over $X_t$ because this can well reflect the bounded performance of agents in reality. We then use the parameter $\alpha_1$ to control the transition between a fully linear relationship and a logistic one. At the two extremes, when $\alpha_1=0$, $X_t$ is a linear transformation of $B_t$ and when $\alpha_1=1$, $X_t$ is a logistic transformation of $B_t$. Similarly, to break Assumption H1, we introduce random perturbation into an agent's effort conversion matrix $\mathcal{E}_t$ and we control the strength of this perturbation with $\alpha_2$. At the two extremes, when $\alpha_2=0$, all agents have the same conversion matrix $\homomatrix$ (as we used throughout our paper), and when $\alpha_2=1$, an agent's effort conversion matrix $\mathcal{E}_t$ follows a multivariate Gaussian distribution. The specific parameters of this distribution depend on the private type of the agent \citep{harris2022strategic}.

\Cref{fig:sensitivity-for-optimum} and \Cref{fig:sensitivity-for-causality} show the respective results, for utility optimisation and for causal parameter learning. In general, the performance declines when assumptions are violated, however our proposed methods do not perform worse than the baselines.

We also study Assumption S1 and provide scatter plots (\Cref{fig:assumptions-S1}) showing the relationship between cBP and $\bm{\theta}_t$ in the single DM case. We make Assumption S1 to simplify our theoretical analysis, but as we demonstrate here, it can be violated in practice and cBP is not fully linear in $\bm{\theta}_t$. However, our experiments on the utility maximisation still display the superior performance of $\thetaao$, despite this violation, hinting that our approach is not sensitive to Assumption S1. This observation, nevertheless, suggests that Assumption S1 should be relaxed, e.g., to a partially linear model. We will pursue this in future work.

\begin{figure*}
\centering
    \begin{subfigure}[t]{0.45\textwidth}
    \centering
    \includegraphics[width=\textwidth]{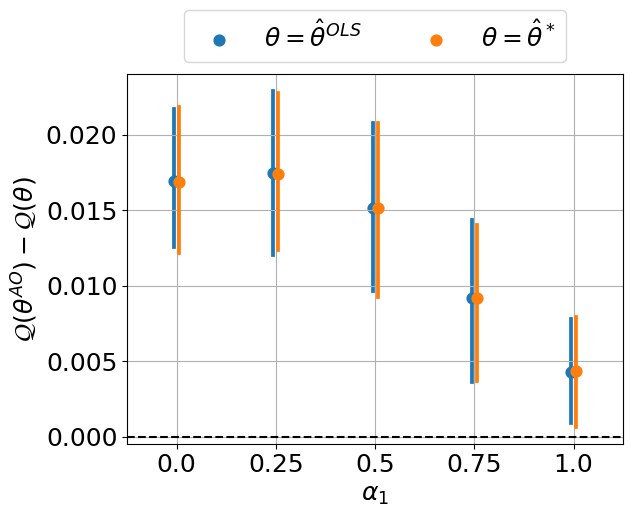}
    \caption{When the linearity assumption for $X_t$ and $B_t$ does not hold.}
    \end{subfigure}
    \hfill
    \begin{subfigure}[t]{0.45\textwidth}
    \centering
    \includegraphics[width=\textwidth]{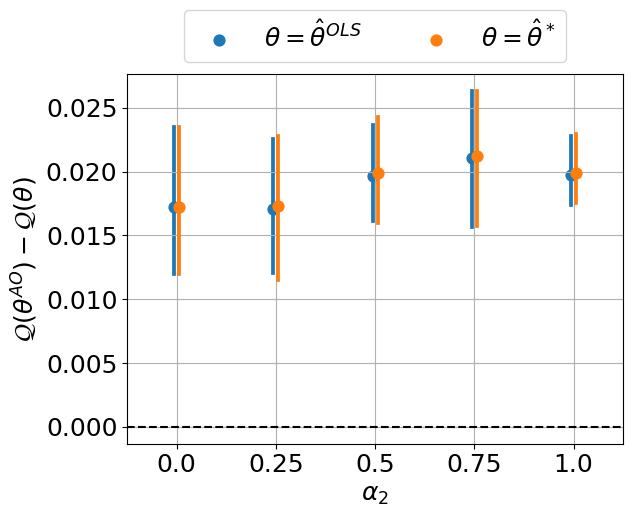}
    \caption{When Assumption H1 does not hold.}
    \end{subfigure}
\caption{Each of the two graphs shows the differences in the utility value $\mathcal{Q}$ between our $\thetaao$ and the other two naive choices of $\bm{\theta}_t$, hence $\mathcal{Q}(\thetaao)-\mathcal{Q}(\bm{\theta})$. Higher values mean better performance of our proposed approach.}
\label{fig:sensitivity-for-optimum}
\end{figure*}

\begin{figure*}
\centering
    \begin{subfigure}[t]{0.45\textwidth}
    \centering
    \includegraphics[width=\textwidth]{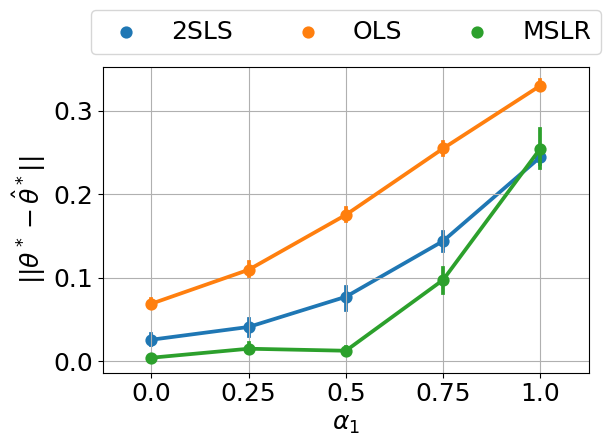}
    \caption{When the linearity assumption for $X_t$ and $B_t$ does not hold.}
    \end{subfigure}
    \hfill
    \begin{subfigure}[t]{0.45\textwidth}
    \centering
    \includegraphics[width=\textwidth]{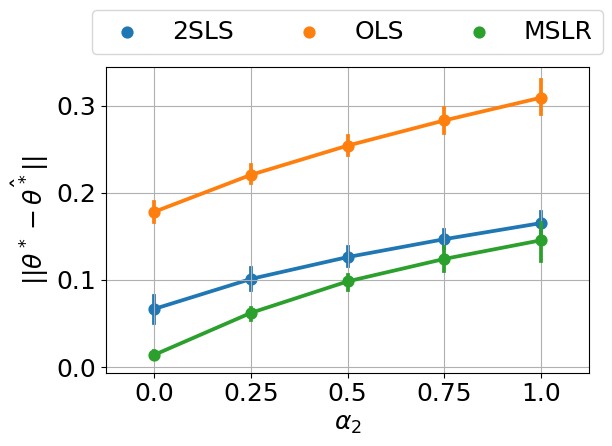}
    \caption{When Assumption H1 does not hold.}
    \end{subfigure}
\caption{Each of the two graphs compares the estimation errors between the three methods for inferring the true causal parameters. Lower values mean better performance.}
\label{fig:sensitivity-for-causality}
\end{figure*}

\begin{figure*}
\centering
    \begin{subfigure}[b]{0.47\textwidth}
    \centering
    \includegraphics[width=\textwidth]{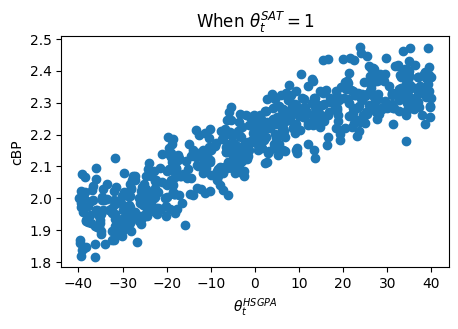}
    \end{subfigure}
    \hfill
    \begin{subfigure}[b]{0.47\textwidth}
    \centering
    \includegraphics[width=\textwidth]{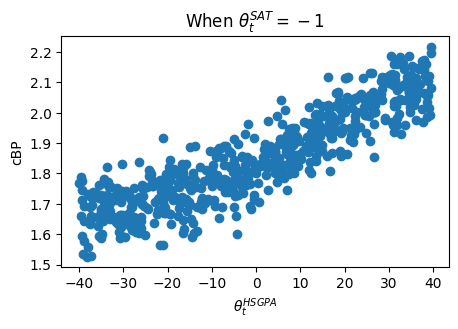}
    \end{subfigure}

    \begin{subfigure}[b]{0.47\textwidth}
    \centering
    \includegraphics[width=\textwidth]{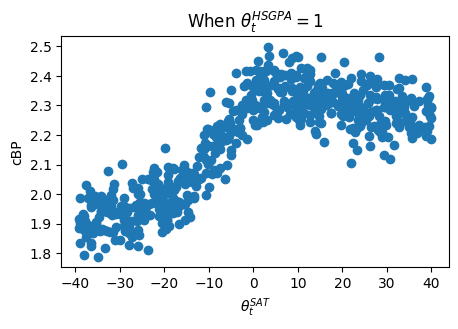}
    \end{subfigure}
    \hfill
    \begin{subfigure}[b]{0.47\textwidth}
    \centering
    \includegraphics[width=\textwidth]{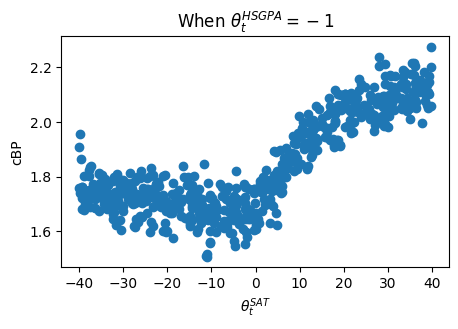}
    \end{subfigure}
\caption{Scatter plots showing the relationship between cBP and $\bm{\theta}_t$ in the single DM case. When either $\theta_t^\text{SAT}$ or $\theta_t^\text{HS GPA}$ is kept fixed.}
\label{fig:assumptions-S1}
\end{figure*}

\subsection{Further Experiments}  \label{apx:exp-further}

In this section, we present some additional results.

\paragraph{\citet{harris2022strategic} algorithm under selection bias.}

\begin{figure}[t]
    \centering
    \includegraphics[width=0.6\textwidth] {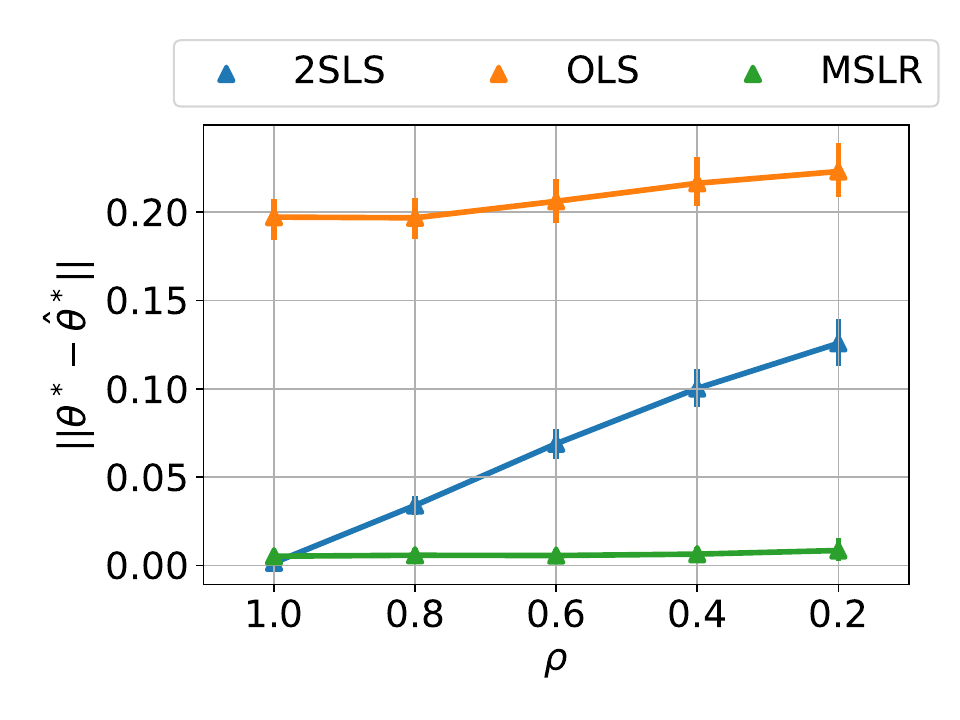}
    \caption{[Lower is better] Estimation errors under decreasing $\rho$ parameter (i.e., under increasing selection bias). The performance of 2SLS deteriorates, whereas our method is robust to the strength of selection bias induced by $\rho$. 
    }
    \label{fig:rho-param}
\end{figure}
\begin{figure}[t!]
     \centering
     \begin{subfigure}[b]{\textwidth}
         \centering
         \includegraphics[width=0.6\textwidth]{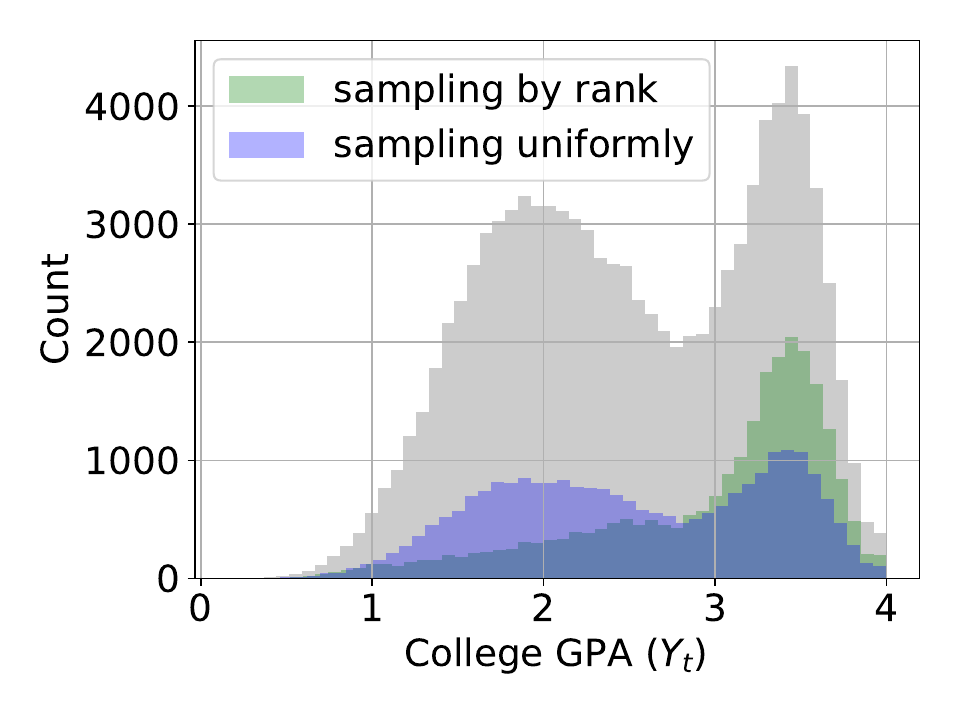}
         \caption{Effect of sampling students either randomly or by percentile rank (see text). Sampling randomly induces a similar distribution to the original one (in gray), while sampling by percentile rank induces a distribution with a larger mode than the original distribution.}
         \label{fig:harris-et-al-selection-dist}
     \end{subfigure}
     \hfill
     \begin{subfigure}[b]{\textwidth}
         \centering
         \includegraphics[width=0.6\textwidth]{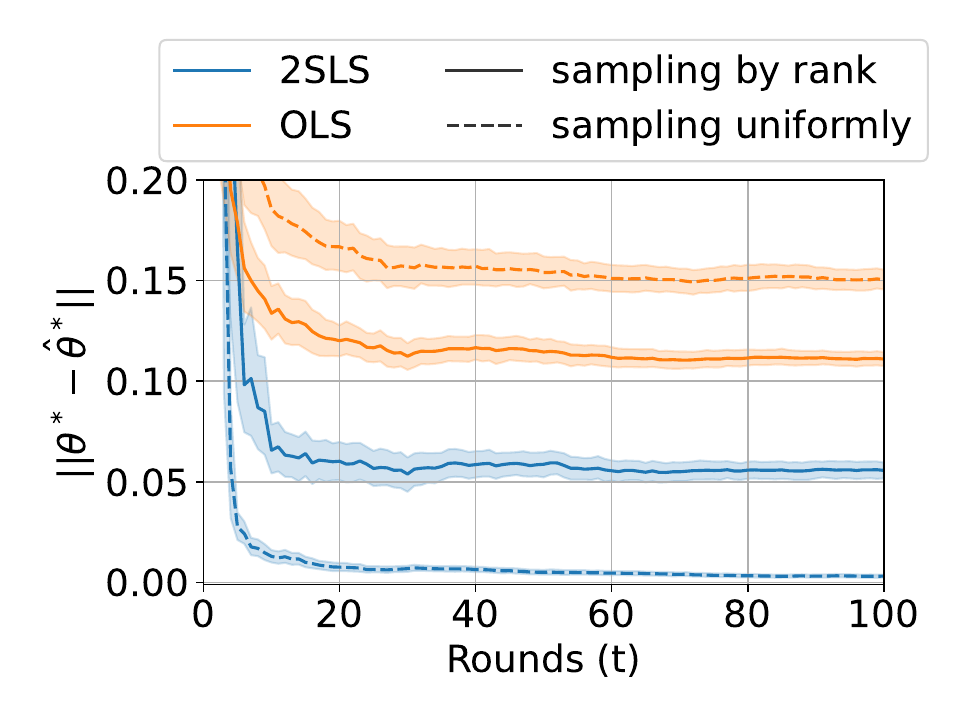}
         \caption{Results of applying 2SLS and OLS on both scenarios from \Cref{fig:harris-et-al-selection-dist}. 2SLS converges to a biased estimate when considering a subsample induced by sampling by percentile rank (i.e., under selection bias), unlike when considering a subsample from a uniform distribution (i.e., no selection bias).  }
         \label{fig:harris-et-al-selection-results}
     \end{subfigure}
        \caption{Illustration that 2SLS produces a biased estimate of the causal parameters $\thetastar$ under selection bias.}
        \label{fig:harris-et-al-selection}
\end{figure}

Firstly, we demonstrate that \citet{harris2022strategic} recovers a biased estimate of the parameters $\thetastar$ under their original settings with the augmentation of a selection variable $W_{it}$. In particular, in our setup, unlike \citet{harris2022strategic}, we assume that (a) all students have the common effort conversion matrix, (b) the decision maker deploys linearly dependent parameter vector $\bm{\theta}_{it}$. We now validate that these assumptions are not necessary: in the absence of aforementioned assumptions, prior work still produces a biased estimate of the parameter vector $\bm{\theta}_{it}$ under selection bias introduced by a selection variable $W_{it}$. 

To this end, we generate $y_{it}$ as per \citet{harris2022strategic}. We now consider two scenarios: (a) Randomly sample $\rho=0.25$ fraction of the students from each round, or (b) sample only the students with predictions $\hat{y}_{it}$ lying in the top $\rho=0.25$ percentile. \Cref{fig:harris-et-al-selection-dist} illustrates the difference in the distribution. We now run 2-stage least squares (2SLS) \citep{harris2022strategic} for both scenarios and demonstrate in \Cref{fig:harris-et-al-selection-results} that the method produces a biased estimate in the latter scenario, unlike in the former scenario. 

\paragraph{Effect of the selection parameter $\rho$ }
Throughout the experiments section, we arbitrarly set $\rho=0.5$  for all environments. We remark that when $\rho=1$, we select all students from the round. We now investigate the impact of $\rho$ on the causal parameter estimation for all methods. To this end, we show in \Cref{fig:rho-param} that the estimation error of 2SLS \citep{harris2022strategic}  deteriorates as the strength of the selection bias increases (i.e., as the $\rho$ parameter \emph{decreases}), whereas our proposed method is robust to the degree of selection bias. 

We note that this variant of ranking selection also preserves the selection statuses of agents under the condition outlined in \Cref{theorem:local-exo-s} and \Cref{theorem:local-exo-m}. This is because this variant is a deterministic function, applied on top of the CDF in \Cref{def:ranking-selection}.

\paragraph{Estimation for $\bm{\theta}^{*, \text{SAT}}$}
\begin{figure}[t!]
    \centering
    \includegraphics[width=0.7\textwidth]{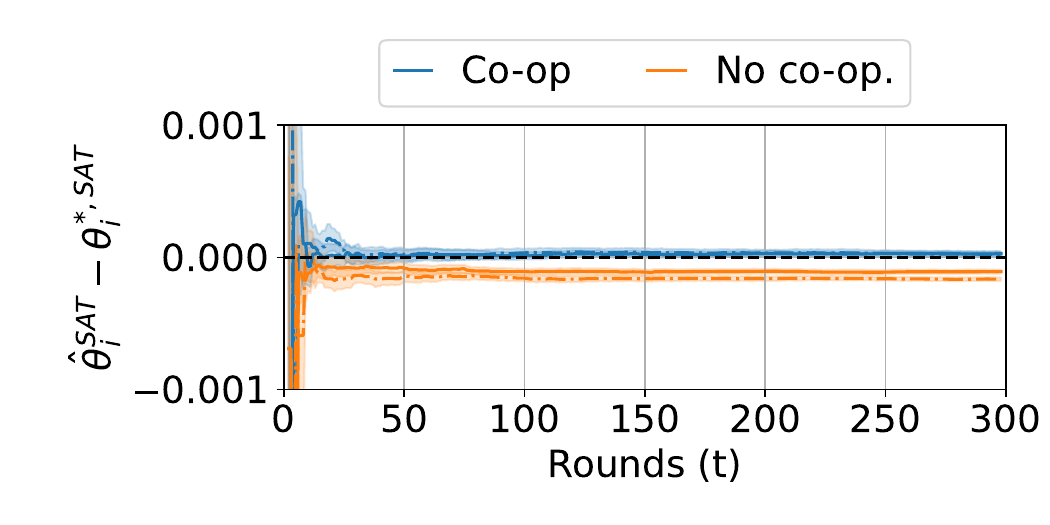}
    \caption{Supplementary result for \Cref{fig:results2}. Bias in the estimated causal effect of $X_{t}^{\text{SAT}}$ on $Y_{it}$ for 2 DMs, each depicted by a different line style.}
    \label{fig:p-vs-np3}
\end{figure}
We demonstrate the supplementary result to \Cref{fig:results2} in \Cref{fig:p-vs-np3}.


\subsection{On normalisation of \texorpdfstring{$\bm{\theta}_{it}$}{theta}}
\label{apx:bounding-theta}
We explain our normalisation scheme used in the experiments. For simplicity, we only describe the single-DM case. Recall that $\mathcal{Q}(\bm{\theta}_t)$ is neither scale-invariant nor bounded in general, Assumption (S2) was introduced to bound $\mathcal{Q}$. This assumption also acts as a constraint that any deployed $\bm{\theta}_t$ must satisfy. 

Our goal is now to find $\arg\max_{\bm{\theta}_t:\|\bm{\theta}_t\|_2\leq\tau}\mathcal{Q}(\bm{\theta}_t)$ where $\tau$ is a generalisation of the threshold in Assumption (S2). This generalisation leads to $\thetaao$ with the property $\|\thetaao\|_2=\tau$. From \Cref{theorem:bounded-optim-s}, it can be seen that $\mathcal{Q}(\bm{\theta}_t)$ can get arbitrarily large depending on the value of $\tau$. Therefore setting $\tau=\|\thetastar\|$ is the only choice for a fair comparison between $\mathcal{Q}(\thetaao)$ and $\mathcal{Q}(\thetastar)$. Similarly, $\tau=\|\thetastarhat\|$ is necessary to compare $\mathcal{Q}(\thetaaohat)$ and $\mathcal{Q}(\thetastarhat)$.

On the other hand, if any $\thetaolshat$ violates the \textit{generalised} Assumption (S2), we scale it down so that $\|\thetaolshat\|=\tau$, otherwise, we do not scale it up. Because unlike $\thetaaohat$, scaling up $\thetaolshat$ does not guarantee an increase in utility $\mathcal{Q}$.

\end{document}